\documentclass{article} 
\usepackage[preprint]{colm2026_conference}

\usepackage{microtype}
\usepackage{hyperref}
\usepackage{url}
\usepackage{booktabs}

\newcommand{\Kun}[2]{\color{green}[Kun]}

\usepackage[most]{tcolorbox}
\usepackage{fvextra}
\usepackage{subcaption}
\usepackage{graphicx}
\usepackage[ruled]{algorithm2e}
\usepackage{enumitem}
\usepackage{amsmath}
\usepackage{amssymb}
\usepackage{amsthm}
\newtheorem{theorem}{Theorem}
\newtheorem{proposition}{Proposition}

\usepackage{lineno}

\definecolor{darkblue}{rgb}{0, 0, 0.5}
\hypersetup{colorlinks=true, citecolor=darkblue, linkcolor=darkblue, urlcolor=darkblue}

\title{Infeasibility Aware Large Language Models for Combinatorial Optimization}

\author{
  Yakun Wang \thanks{Part of the work was done while the author was at Johnson \& Johnson.} \\
  Lehigh University \\
  Bethlehem, PA 18015, USA \\
  \texttt{yaw220@lehigh.edu} \\
  \And
  Min Chen, Zeguan Wu \& Junyu Liu \\
  Department of Computer Science \\
  University of Pittsburgh \\
  Pittsburgh, PA 15260, USA \\
  \texttt{\{mic380, zew79, junyuliu\}@pitt.edu} \\
  \AND
  Sitao Zhang \& Zhenwen Shao \\
  Johnson \& Johnson \\
  New Brunswick, NJ 08933, USA \\
  \texttt{\{SZhang78, ZShao5\}@its.jnj.com}
}

\begin{document}

\ifcolmsubmission
\linenumbers
\fi

\maketitle

\begin{abstract}
Large language models (LLMs) are increasingly explored for NP-hard combinatorial optimization problems, but most existing methods emphasize feasible-instance solution generation and do not explicitly address infeasibility detection. We propose an infeasibility-aware framework that combines certifiable dataset construction, supervised fine-tuning, and LLM-assisted downstream search. For the minor-embedding problem, we introduce a new mathematical programming formulation together with provable zero-phase infeasibility screening, which enables scalable construction of training instances labeled either as feasible with structured certificates or as certifiably infeasible. Using training data generated through this exact optimization pipeline, we show that an 8B-parameter LLM can be fine-tuned to jointly perform solution generation and infeasibility detection. We further utilize LLM outputs as warm starts for downstream local search, providing a practical way to accelerate optimization even when the LLM outputs are imperfect. Experiments show that our fine-tuned model improves overall accuracy by up to 30\% over GPT-5.2; meanwhile LLM-guided warm starts provide up to $2\times$ speedup compared with starting from scratch in downstream local search. 
\end{abstract}

\section{Introduction}
Large language models (LLMs) have rapidly progressed beyond traditional natural language processing (NLP) tasks and are increasingly viewed as general-purpose reasoning and decision-making agents~\citep{yao2022react,huang2022language}. This evolution has sparked growing interest in their potential to address NP-hard combinatorial optimization (CO) problems~\citep{yang2023large, xiao2023chain,optimus2023}, which arise ubiquitously in real-world applications~\citep{da2025large, bengio2021machine}.

Recent research on LLMs for combinatorial optimization falls into two paradigms: The first leverages LLMs as assistive tools to generate executable code, discover heuristics, or interact with mathematical optimization solvers~\citep{xiao2023chain,jiang2024llmopt,huang2025orlm,ye2024reevo}. While effective, these approaches still require substantial domain expertise; Another paradigm explores end-to-end LLM-based solvers that address CO problems directly in a zero-shot or few-shot setting, without relying on problem-specific templates or external solver calls~\citep{yang2023large,jiang2025large}. However, the applicability of these approaches to more complex and practically relevant CO problems remains largely unexplored. Moreover, existing end-to-end LLM solvers are not reliable enough since they typically lack infeasibility detection and do not incorporate systematic post-processing mechanisms to repair or validate incorrect solutions. A further challenge lies in the construction of training datasets: For some classical CO problems, well-established algorithms can generate reliable ground truth solutions at scale. In contrast, for emerging and structurally complex problems such as minor-embedding, current optimization solvers are limited in producing solutions or infeasibility certificates within reasonable time~\citep{bernal2020integer}, making supervised learning particularly challenging.

Motivated by these gaps, we propose an infeasibility-aware framework for applying LLMs to NP-hard combinatorial optimization problems. The framework consists of three stages: (i) construction of training datasets containing both feasible instances and certifiably infeasible instances, (ii) supervised fine-tuning of a small language model to jointly predict structured solutions or infeasibility, and (iii) LLM-guided downstream search, where model outputs are used to initialize local search procedures. This design allows the LLM to function not only as a predictor, but also as a practical component in a broader optimization pipeline. We focus on the minor-embedding problem as our primary case study because it is a fundamental task in quantum computing~\citep{choi2008minor} where even determining feasibility is NP-hard~\citep{lobe2021minor}. Its intricate constraints often exceed the capabilities of classical solvers like Gurobi~\citep{gurobi} or OR-Tools~\citep{ortools}, making it an ideal testbed for LLM-based initialization. To demonstrate that the resulting workflow is not specific to minor-embedding, we also evaluate it on $k$-coloring as a second NP-hard problem.

Our contributions are summarized as follows:
\begin{itemize}
\item We propose an infeasibility-aware workflow for LLM-based combinatorial optimization that combines three components: the construction of training datasets containing both feasible and certifiably infeasible instances, supervised fine-tuning for structured solution prediction and infeasibility detection, and LLM-assisted initialization of downstream local search.

\item For the minor-embedding problem, we develop a new dataset construction methodology based on mathematical programming formulations and infeasibility certification, enabling scalable generation of labeled instances for a problem where both structured feasible solutions and reliable infeasibility labels are difficult to obtain.

\item We show that a fine-tuned 8B-parameter LLM can jointly perform solution generation and infeasibility detection, and its outputs provide effective warm starts for local search algorithm like \emph{Feasibility Jump}~\citep{luteberget2023feasibility}. Across minor-embedding and $k$-coloring, the resulting approach improves overall accuracy by up to 30\% over a general-purpose LLM and yields up to $2\times$ speedup compared with local search starting from scratch.
\end{itemize}

\section{Related work}

This section reviews prior work most relevant to our setting: LLMs for combinatorial optimization and reasoning under hard constraints, and existing methods for the minor-embedding problem.

\paragraph{LLMs for combinatorial optimization.}
Recent work has explored several roles for LLMs in combinatorial optimization. One line of research uses LLMs to generate reusable solver components, such as heuristics, search operators, or executable code, as in FunSearch~\citep{funsearch2024}, ReEvo~\citep{ye2024reevo}, and subsequent reinforcement-learning-based extensions~\citep{surina2025algorithm}. A second line studies LLMs as auto-formulation engines that translate natural-language descriptions into mathematical models, solver code, or optimization pipelines, including Chain-of-Experts~\citep{xiao2023chain}, OptiMUS~\citep{optimus2023,optimus03_2024}, tree-search-based formulation methods~\citep{jiang2024llmopt}, and ORLM~\citep{huang2025orlm}. A third line evaluates LLMs as solvers that directly produce structured answers for problem instances, either by zero-shot generation or through iterative refinement~\citep{jiang2025large,yang2023large,liu2024evolutionary}. Related hybrid methods further combine language models with symbolic or logic-based reasoning modules to improve correctness on constrained tasks~\citep{ye2023satlm,zhang2024dila}. 
Despite these advances, prior work primarily emphasizes candidate generation, code synthesis, or solver-aided reasoning over infeasibility detection~\citep{da2025large}. While some recent work uses LLMs to diagnose faulty optimization models~\citep{chen2024diagnosing}, directly predicting infeasibility for combinatorial instances remains largely unexplored.
Recent boolean satisfiability problem (SAT) oriented studies also suggest that feasibility reasoning remains challenging for current language models, especially on hard or UNSAT-style instances~\citep{pan2024can,hazra2025have,wei2025satbench}. Unlike solver-aided reasoning methods that depend on inference-time symbolic backends, our framework uses exact optimization strictly offline for data construction. This difference is especially relevant for structured combinatorial optimization, where infeasibility often depends on problem-specific certificates rather than generic logical encodings that can be unnatural or unscalable in practice.

\paragraph{Minor-embedding.}
The minor-embedding problem is a fundamental preprocessing step in quantum annealing (QA)~\citep{kadowaki1998quantum}, requiring the logical problem graph to be mapped into the sparse hardware graph of a quantum processing unit by representing each logical variable as a connected chain of physical qubits~\citep{choi2008minor,choi2011minor}. This problem is NP-hard~\citep{lobe2021minor}, and embedding quality directly impacts solution quality, as longer chains increase susceptibility to chain breaks during annealing~\citep{fang2020minimizing}. The dominant practical tool is the \texttt{minorminer} heuristic~\citep{cai2014practical}, complemented by polynomial-time clique embedding for fully connected graphs~\citep{boothby2016fast} and various alternatives including layout-aware~\citep{zbinden2020embedding}, simulated-annealing-based~\citep{sugie2021minor}, and reinforcement-learning-based approaches~\citep{nembrini2025minor}. 
While effective, these heuristics share a critical limitation: they provide neither optimality guarantees nor certificates of infeasibility. Consequently, a failed embedding attempt cannot distinguish between instances that are truly infeasible and those that merely fall outside the scope of the heuristic algorithm. This ambiguity creates a costly bottleneck in QA workflows, where quantum processing unit (QPU) access is limited and expensive. While exact methods based on algebraic geometry~\citep{dridi2018algebraic} and integer programming~\citep{bernal2020integer} can certify infeasibility, they remain tractable only for small instances. Our work bridges this gap: we use exact solvers to generate provably labeled training data for both feasible embeddings and infeasibility certificates, then train an LLM to generalize these patterns, enabling fast and reliable infeasibility screening at scale.

\section{Methodology}
This section describes the proposed framework and its core components. Figure~\ref{fig:framework} illustrates the architectural overview of the system. We first present the dataset construction methodology in Section~\ref{sec:dataset}. We then detail the supervised fine tuning process in Section~\ref{sec:SFT}. Finally, we discuss inference and post processing in Section~\ref{sec:FJ}.
\begin{figure}[htpb]
\centering
\includegraphics[width=0.8\linewidth]{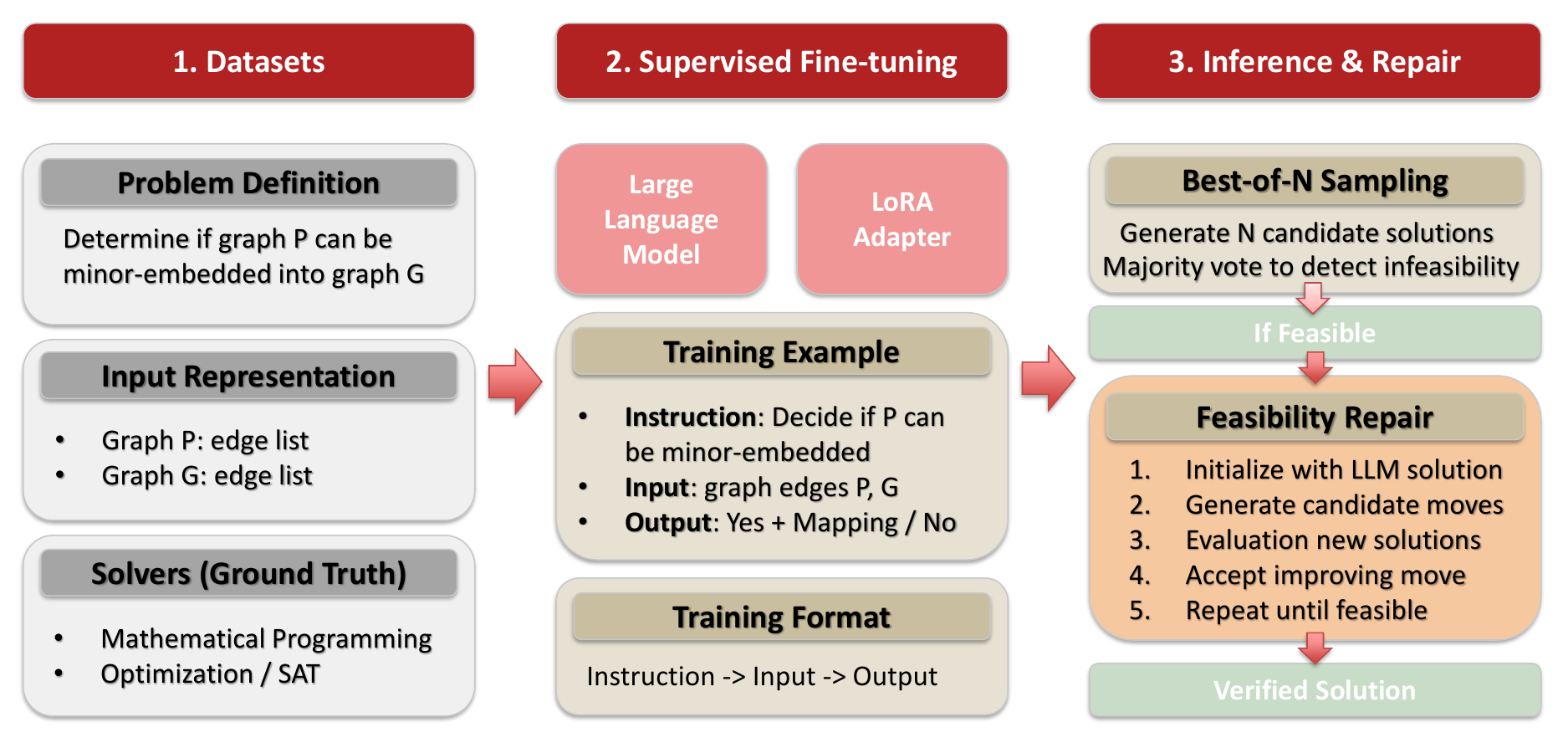}
\caption{Framework.}
\label{fig:framework}
\end{figure}
\subsection{Mathematical programming for dataset generation}
\label{sec:dataset}
Training the LLM requires high-quality ground truth labels for each problem instance. Mathematical programming is a standard approach for combinatorial optimization because it provides optimality guarantees or certificates of infeasibility when processed by solvers such as Gurobi or OR-Tools. To facilitate this, we provide a formal problem definition and the mathematical programming formulation for the tasks considered in this study.

\subsubsection{Mathematical programming formulation for minor-embedding}
\label{sec:minor-embedding-model}
In this section, we present a new mathematical programming formulation for the minor-embedding problem, which we use to construct ground-truth solutions for dataset instances. In addition, we propose a property that enables fast infeasibility detection for a subset of instances.

We consider the minor-embedding of a problem graph $P=(V_P,E_P)$ into a hardware graph $G=(V_G,E_G)$. A minor-embedding assigns to each problem vertex $i\in V_P$ a connected set of hardware vertices (called a \emph{chain}) such that chains assigned to distinct problem vertices are pairwise disjoint, and every problem edge is realized by at least one hardware edge between the corresponding chains.

\paragraph{Admissible chains.}
As noted in~\citet{fang2020minimizing, bernal2020integer}, longer chains are more susceptible to chain breaks; therefore, limiting chain size is good practice. We exploit this observation to derive a compact integer programming model. Specifically, we assume that a family $\mathcal{C}\subseteq 2^{V_G}$ of admissible connected chains in $G$ has been enumerated in advance, optionally with a size bound $|C|\le L$. For each chain $C\in\mathcal{C}$, let $s(C)=|C|$, and define $\Gamma(C):=\{D\in\mathcal{C}: \exists\, \{u,v\}\in E_G \text{ with } u\in C,\ v\in D\},$ the set of chains adjacent to $C$ in the hardware graph. Let $A_P:=\{(i,k),(k,i): \{i,k\}\in E_P\}$ denote the directed arc set induced by the edges of the problem graph.

\paragraph{Decision variables.}
For each problem vertex $i\in V_P$ and chain $C\in\mathcal{C}$, let $x_{iC}\in\{0,1\}$ indicate whether vertex $i$ is assigned to chain $C$.

\paragraph{Formulation.}
The minor-embedding problem is to minimize \eqref{eq:objective-main} subject to \eqref{eq:feas-assign-main}--\eqref{eq:feas-adj-main}:
\begin{align}
    \min \textstyle & \sum_{i,C} s(C)x_{iC} \label{eq:objective-main} \hspace{-2em} \\
    \text{s.t. } &\textstyle \sum_{C} x_{iC} = 1  && \forall i \in V_P \label{eq:feas-assign-main} \\
    &\textstyle \sum_{i, C:j\in C} x_{iC} \le 1 && \forall j \in V_G \label{eq:feas-disjoint-main} \\
    &x_{iC} \le \textstyle \sum_{D \in \Gamma(C)} x_{kD} && \forall (i,k) \in A_P, C \in \mathcal{C} \label{eq:feas-adj-main}
\end{align}

Constraints \eqref{eq:feas-assign-main} ensure unique chain assignment, \eqref{eq:feas-disjoint-main} enforce disjointness, and \eqref{eq:feas-adj-main} maintain adjacency between problem and hardware graphs.

For a prescribed chain family $\mathcal{C}$, the formulation uses $O(|V_P||\mathcal{C}|)$ binary variables and $O(|V_P|+|V_G|+|E_P||\mathcal{C}|)$ constraints. When $\mathcal{C}$ consists of all connected chains of size at most $L$, we have $|\mathcal{C}| = O(|V_G|^L)$ in the worst case, implying a polynomial-size formulation for fixed $L$. In sparse hardware graphs, the number of connected chains is often much smaller in practice. Nevertheless, the existence of a polynomial-size formulation for fixed $L$ does not make the problem easy: even with the admissible chain family given explicitly, the feasibility version of the problem remains computationally hard. We formalize this below.

\begin{theorem}
\label{thm:meec-feas-npc}
Consider a problem graph $P=(V_P,E_P)$, a hardware graph $G=(V_G,E_G)$, and an explicitly listed family $\mathcal{C}$ of connected subsets of $V_G$, with $\{v\} \in \mathcal{C}$ for all $v \in V_G$. Deciding whether there exists a assignment satisfying constraints \eqref{eq:feas-assign-main}--\eqref{eq:feas-adj-main} (denoted as \emph{MEEC-Feas}) is NP-complete.
\end{theorem}
The proof is given in Appendix~\ref{appendix:proofs}.

\paragraph{Zero-phase infeasibility detection.}
Although the mathematical programming formulation can certify infeasibility and recover an embedding when one exists, solving it can be expensive when generating large benchmark datasets. We therefore apply a \emph{zero-phase infeasibility detection} step before chain enumeration and optimization. Instances rejected at this stage are provably infeasible, while instances that pass are simply not ruled out and are forwarded to the solver. In the following, let $\Delta_H:=\max_{v\in V_H}\deg_H(v)$, and assume that every chain has size at most $L$.
\begin{proposition}[Zero-phase infeasibility certificate]
\label{prop:zero-phase-main}
Suppose a problem graph $P=(V_P,E_P)$ admits a minor-embedding into $H$ with maximum chain size $L$. Then the following necessary conditions hold:
\begin{equation}
\deg_P(i) \le L(\Delta_H-2)+2 \quad \forall i\in V_P, \label{eq:zp-degree-bound-main}
\end{equation}

\noindent\begin{minipage}{.5\textwidth}
\begin{equation}
  \sum_{i\in V_P} s_i^{\min} \le |V_H|, \label{eq:zp-node-bound-main}
\end{equation}
\end{minipage}%
\begin{minipage}{.5\textwidth}
\begin{equation}
  |E_P|+\sum_{i\in V_P}(s_i^{\min}-1) \le |E_H|, \label{eq:zp-edge-bound-main}
\end{equation}
\end{minipage}

where $s_i^{\min}:=
\max\!\left(1,\left\lceil\frac{\deg_P(i)-2}{\Delta_H-2}\right\rceil\right)
\quad (\Delta_H\ge 3)$ is a lower bound on the chain size required for problem vertex $i$.
Consequently, violation of any one of \eqref{eq:zp-degree-bound-main}--\eqref{eq:zp-edge-bound-main} certifies infeasibility.
\end{proposition}

Proposition~\ref{prop:zero-phase-main} follows from the fact that each problem vertex is mapped to a connected chain in the hardware graph and that a chain of size $s$ in a graph of maximum degree $\Delta_H$ can support at most $s(\Delta_H-2)+2$ external adjacencies \citep{choi2008minor,fang2020minimizing}. A complete proof is provided in Appendix~\ref{appendix:proofs}.

We use Proposition~\ref{prop:zero-phase-main} as an initial infeasibility screening step. We leave more details of implementation in Appendix~\ref{appendix:minor-embedding1}. We note that a mathematical programming formulation for minor-embedding was previously proposed in~\citet{bernal2020integer}. We present the details of that formulation and compare it with ours in Appendix~\ref{appendix:minor-embedding2}.

\subsubsection{Mathematical programming formulation for k-coloring}
We also employ the graph coloring problem as a benchmark for our framework. This problem is particularly instructive because it encompasses two distinct variants: a decision-based version and an optimization-based version. While both versions are computationally difficult (NP-complete/NP-hard) \citep{karp2009reducibility}, our numerical experiments demonstrate that instances involving infeasibility are significantly more challenging for LLMs to resolve directly. We leave details of $k$-coloring problem to Appendix~\ref{appendix:k-coloring}.

\subsection{Supervised fine-tuning}
\label{sec:SFT}
We perform parameter-efficient supervised fine-tuning using Low-Rank Adaptation (LoRA)~\citep{hu2022lora} on models from the Llama family~\citep{llama31herd}. LoRA adapts the pretrained weight matrix \(W\) as \(W' = W + BA\), where \(A \in \mathbb{R}^{r \times k}\) and \(B \in \mathbb{R}^{d \times r}\) are trainable low-rank matrices, \(W\) remains frozen, and \(r \ll \min(d,k)\) denotes the LoRA rank. This substantially reduces the number of trainable parameters and memory cost.  We adopt the standard SFT pipeline: Let \(x\) denote the input instance and \(y = (y_1,\dots,y_T)\) denote the target output sequence. We train the model using the standard autoregressive next-token prediction objective $\mathcal{L}_{\mathrm{SFT}}
= - \sum_{t=1}^{T} \log p_{\theta}\!\left(y_t \mid x, y_{<t}\right),$
where \(\theta\) denotes the trainable LoRA parameters, while the pretrained backbone parameters remain fixed. Notably, our key design choice about the target format is that we adopt a unified and compact output representation in which the model directly emits either a machine-verifiable solution certificate for feasible instances or an explicit infeasibility declaration for infeasible instances (see Appendix~\ref{appendix:prompt} for details). No intermediate reasoning or chain-of-thought is required. This design encourages concise outputs, enabling faster inference and reducing token usage. The unified format allows the model to jointly learn feasibility recognition and solution generation within a single autoregressive process. To further stabilize learning, we enforce strict consistency in formatting, ordering, and solution representation across the dataset.

\subsection{Inference and feasibility repairing}
\label{sec:FJ}
Since autoregressive generation does not guarantee structural correctness, the model may still output invalid solutions even when feasibility is predicted correctly. To address this issue, we propose two mechanisms: Best-of-$N$ inference and feasibility jump.
\paragraph{Best-of-$N$ inference.}
We adapt Best-of-$N$ strategy~\citep{wang2022self} to select the highest quality solution certificates. 
At inference time, we sample the fine-tuned model $N$ times to obtain candidates $\hat y^{(1)}, \dots, \hat y^{(N)}$. A verification module identifies valid certificates by setting $v_n = 1$ for feasible solutions and $v_n = 0$ otherwise. If at least one valid solution exists, we return the candidate $n^\star$
such that
$
n^\star \in \arg\min_{n: v_n=1} f(\hat y^{(n)};x)
$
.
If no valid certificate is found ($\sum v_n = 0$), we predict the feasibility status $\hat c(x)$ via a majority vote over the feasibility claims $c_n \in \{\texttt{yes}, \texttt{no}\}$ extracted from each sample.
:
$$
\hat c(x)=
\begin{cases}
\texttt{no}, & \text{if } \sum_{n=1}^N \mathbf{1}\{c_n=\texttt{no}\} > N/2,\\
\texttt{yes}, & \text{otherwise}.
\end{cases}
$$
This asymmetric policy prioritizes verified certificates as decisive evidence and utilizes majority voting only as a fallback for cases where verification fails.

\paragraph{Feasibility Jump.}

Feasibility Jump (FJ) \citep{luteberget2023feasibility} is a problem-agnostic primal heuristic for Mixed Integer Linear Programs (MILPs) (see Appendix~\ref{appendix:milp_background}). While FJ searches directly for solutions from arbitrary assignments, it lacks a mechanism to warm start from high-quality solutions to accelerate convergence. Our framework bridges this gap by using LLMs to provide initial configurations.

Starting from an initial assignment, FJ iteratively performs ``jumps" (optimal single-variable changes) to minimize the weighted violation objective. If FJ hits a local minimum, it increases the weights of violated constraints to escape. We parse LLM outputs and check them for feasibility: feasible solutions are accepted directly, while infeasible ones serve as the starting point for FJ. This synergy is particularly effective because LLM outputs are often structurally informative yet imperfect, precisely the type of input FJ is designed to repair. Appendix~\ref{appendix:fj} provides a detailed description of the repair procedure and update rules.

\section{Experiments}
We fine-tune the Llama 3.1 8B model~\citep{llama31herd} as our core generative solver. For each problem type, we construct 10,000 labeled instances. We use balanced datasets with around 50\% satisfiable (SAT) and 50\% unsatisfiable (UNSAT) instances. Training is performed with Unsloth~\citep{unsloth}, and inference uses top-$p$ sampling with $p=0.7$. We compare against GPT-5.2, GPT-5.2-Reasoning, DeepSeek-V3.2, Llama-3.3-70B, and Grok-4.1-Fast. Detailed training hyperparameters and dataset construction procedures are provided in the Appendix~\ref{appendix:experiment setting}. We also compare token usage and show that our fine-tuned model is more efficient than the benchmark models; see Appendix~\ref{appendix: token_usage}.

\subsection{Main results}
\paragraph{Minor-embedding.}

\begin{table}[htbp]
\small
\begin{center}
\begin{tabular}{ccccc}
\toprule
 Model             & Inf. Det.               & SAT Acc.               & UNSAT Acc.    & Time (s)           \\
\midrule
Deepseek-V3.2      & 55.1\%               &   74.2\%             &  36.0\%             &$17.37\pm71.72$ \\
Llama3.3-70B       & 73.0\%            &    79.0\%          &   67.0\%       &  $3.72\pm5.09$  \\
Grok-4.1-Fast (Reasoning) & 71.2\%            &    42.4\%          &   \textbf{100\%}       &  $114.40\pm107.72$  \\
GPT5.2             & 78.5\%                  & 61.2\%                    & 95.8\%        &      $8.67\pm5.55$                     \\
GPT5.2-Reasoning   & 77.5\%                  & 55.0\%                   & \textbf{100\%}           &  $49.68\pm54.36$     \\
Llama-8b SFT BestOf1 & 98.5\%                  & 98.0\%                   & 99.0\%               &  $2.30\pm2.66 $          \\
Llama-8b SFT BestOf8 & \textbf{99.9\%}                  & \textbf{100\%}                   & 99.8\%             &   $ 3.86\pm3.68 $  \\
\bottomrule
\end{tabular}
\end{center}
\caption{Evaluation of infeasibility detection accuracy (Inf. Det.), breakdown accuracy for satisfiable (SAT) and unsatisfiable (UNSAT) instances on the minor-embedding problem. Bold values indicate the best performance for each metric.}
\label{tab:minor-embedding-main}
\end{table}

\begin{figure}[htbp]
    \centering

    \begin{subfigure}[b]{0.325\linewidth}
        \centering
        \includegraphics[width=\linewidth]{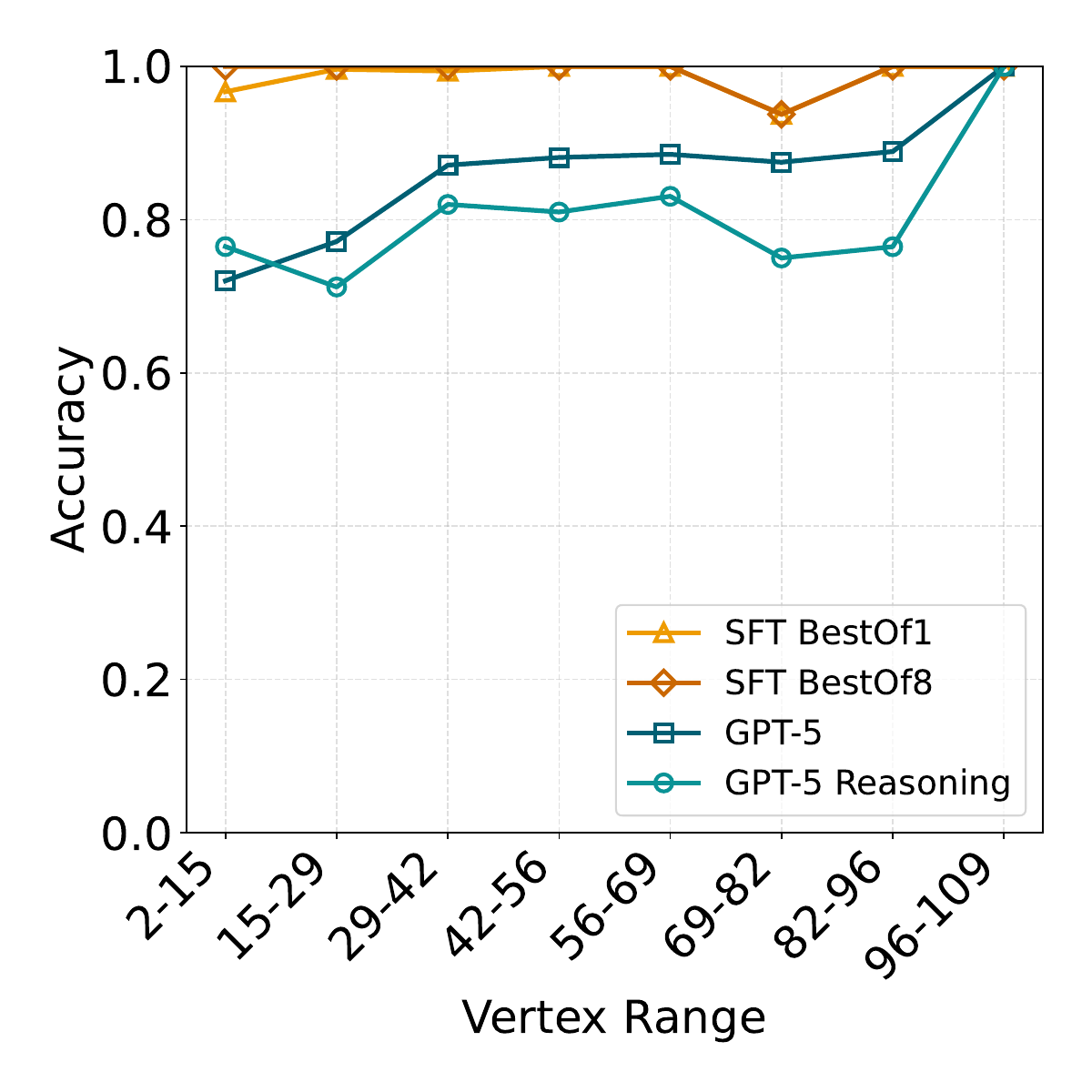}
        \caption{Overall}
        \label{fig:embedding-overall}
    \end{subfigure}
    \hfill
    \begin{subfigure}[b]{0.325\linewidth}
        \centering
        \includegraphics[width=\linewidth]{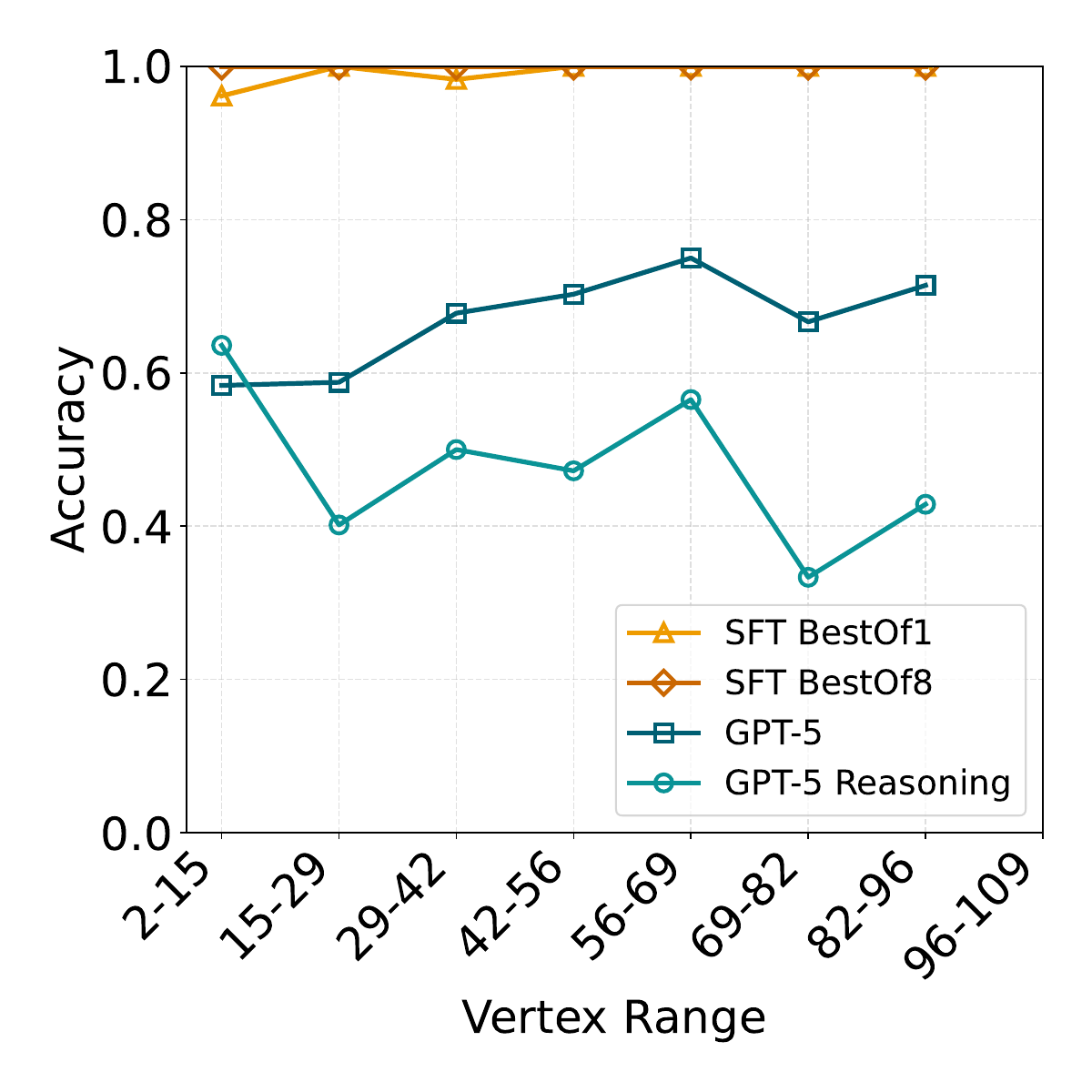}
        \caption{SAT}
        \label{fig:embedding-yes}
    \end{subfigure}
    \hfill
    \begin{subfigure}[b]{0.325\linewidth}
        \centering
        \includegraphics[width=\linewidth]{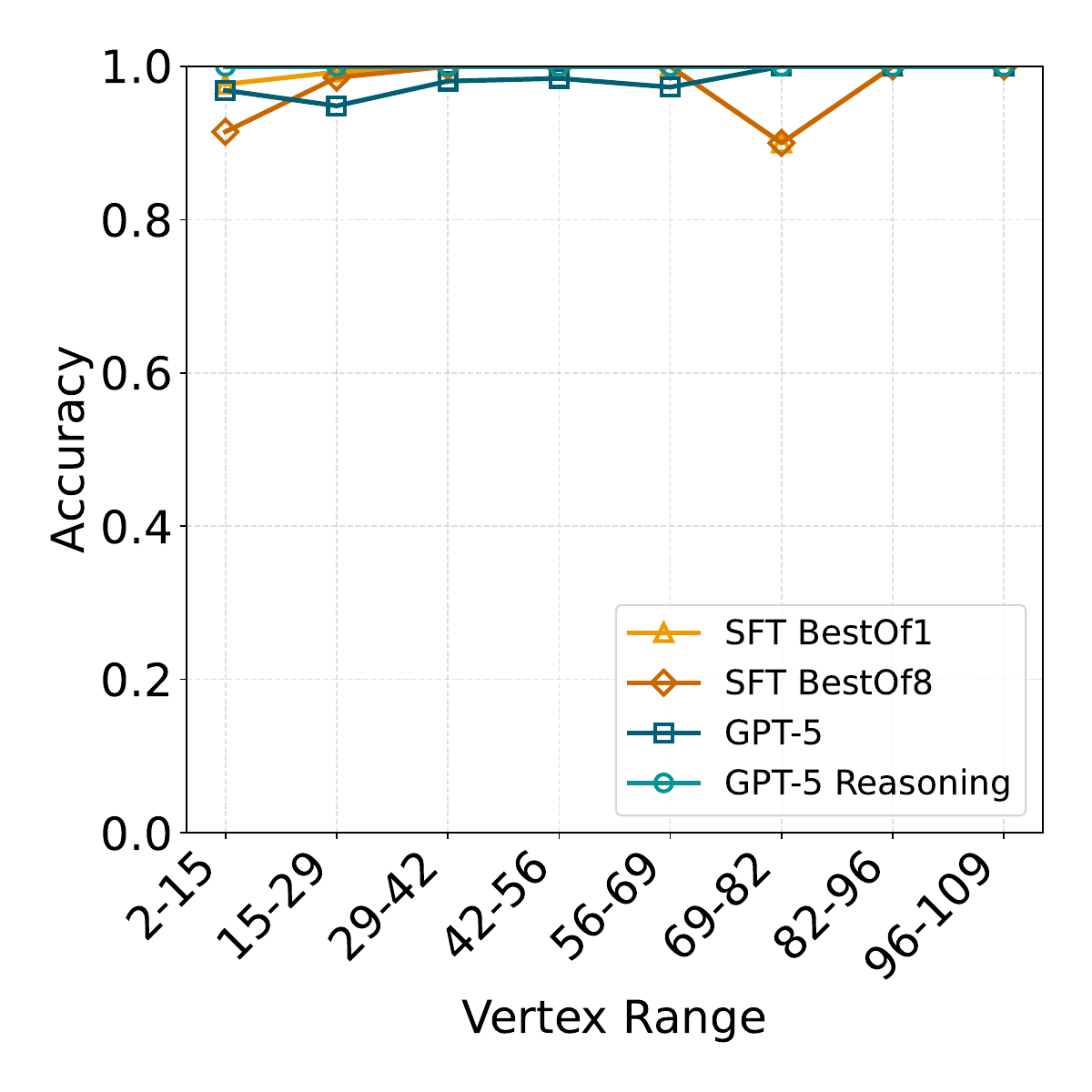}
        \caption{UNSAT}
        \label{fig:embedding-no}
    \end{subfigure}
    
    \caption{Accuracy breakdown by the problem graph vertices number. Subplots show results for (a) the full dataset, (b) SAT instances only, and (c) UNSAT instances only. }
    \label{fig:embedding-accuracy}
\end{figure}
Table~\ref{tab:minor-embedding-main} and Figure~\ref{fig:embedding-accuracy} summarize the minor-embedding results. Our fine-tuned model clearly outperforms all baselines while remaining balanced across SAT and UNSAT instances. \textsc{Llama-8B SFT} already achieves 98.5\% overall accuracy with Best-of-1, including 98.0\% SAT accuracy and 99.0\% UNSAT accuracy. Best-of-8 further improves performance to 99.9\% overall accuracy, with 100\% SAT accuracy and 99.8\% UNSAT accuracy.

In contrast, the frontier models are much more asymmetric. GPT-5.2 and GPT-5.2-Reasoning perform well on UNSAT instances (95.8\% and 100\%), but are much weaker on SAT instances (61.2\% and 55.0\%), indicating a bias toward predicting infeasibility rather than constructing valid embeddings. Figure~\ref{fig:embedding-accuracy} shows that this advantage remains consistent across instance sizes.

\paragraph{Graph coloring.}
Table~\ref{table:3-coloring} and Figure~\ref{fig:3-coloring-accuracy} summarize the performance on the 3-coloring problem. 

\begin{table}[htbp]
\small
\begin{center}
\begin{tabular}{ccccc}

\toprule
 Model             & Inf. Det.               & SAT Acc.               & UNSAT Acc.        & Time (s)       \\
\midrule
Deepseek-V3.2      & 48.1\%                  & 26.9\%                    & 71.7\%               & $11.75\pm56.95$\\
Llama3.3-70B       & 47.7\%                  & 1.3\%                     & 99.2\%               & $1.08\pm0.92$ \\
Grok-4.1-Fast (Reasoning)  & 60.0\%                & 23.5\%                    & 99.2\%               &   $111.42\pm61.20$ \\
GPT5.2             & 50.4\%                  & 5.89\%                    & 99.79\%              &   $7.92\pm10.14$                 \\
GPT5.2-Reasoning   & 62.3\%                  & 28.33\%                   & \textbf{100}\%        &  $42.98\pm55.19$         \\
Llama3.1-8b SFT BestOf1 & 83.1\%                  & 83.84\%                   & 82.28\%                &   $11.22\pm11.58$                \\
Llama3.1-8b SFT BestOf8 & \textbf{89.2\%}                  & \textbf{95.82\%}                   & 81.86\%           &     $16.55\pm12.58$     \\
\bottomrule
\end{tabular}
\end{center}
\caption{Evaluation of infeasibility detection (Inf. Det.), breakdown accuracy for satisfiable (SAT) and unsatisfiable (UNSAT) instances on the 3-coloring problem. Bold values indicate the best performance for each metric.}
\label{table:3-coloring}
\end{table}

\begin{figure}[htbp]
    \centering

    \begin{subfigure}[b]{0.325\linewidth}
        \centering
        \includegraphics[width=\linewidth]{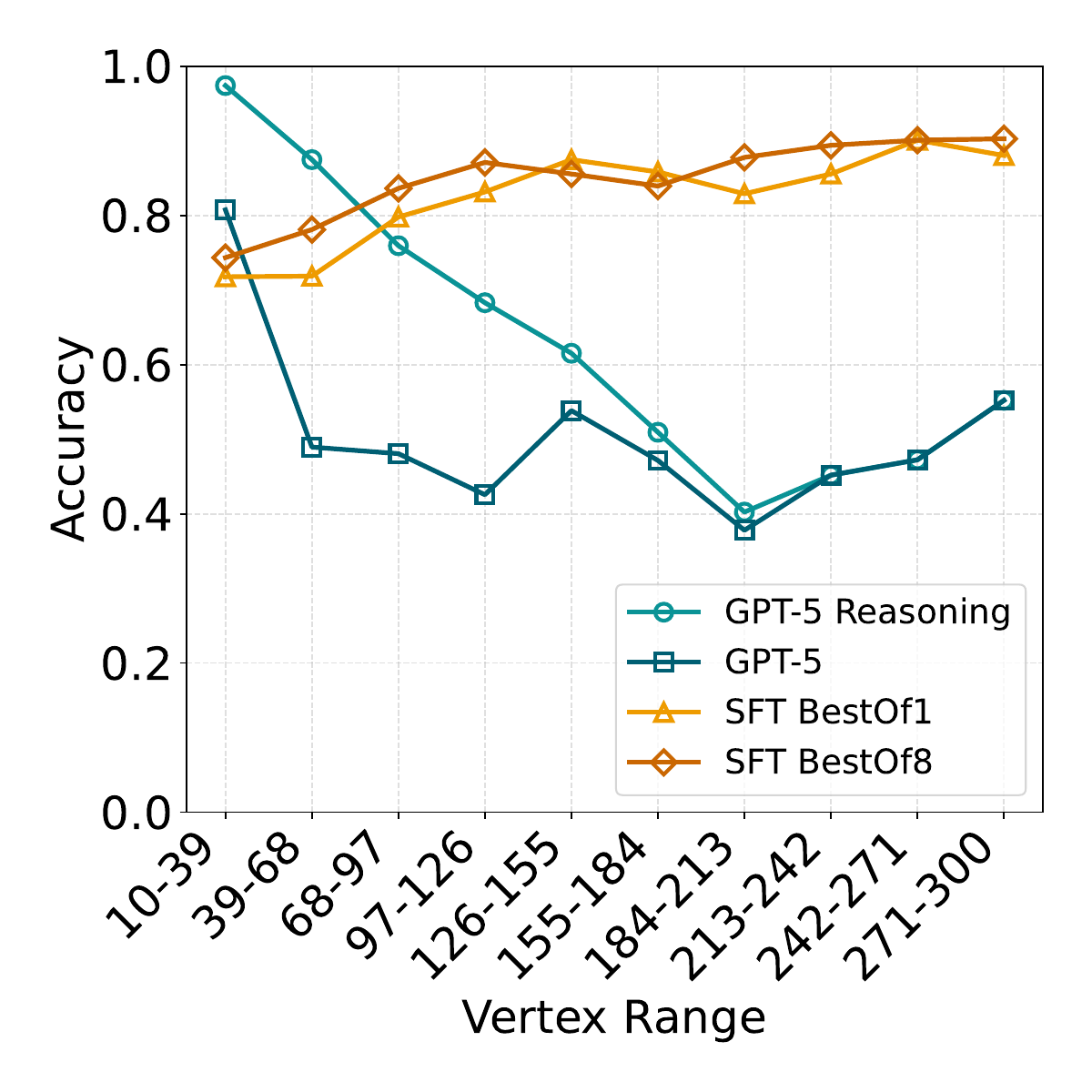}
        \caption{Overall}
        \label{fig:3-coloring-overall}
    \end{subfigure}
    \hfill
    \begin{subfigure}[b]{0.325\linewidth}
        \centering
        \includegraphics[width=\linewidth]{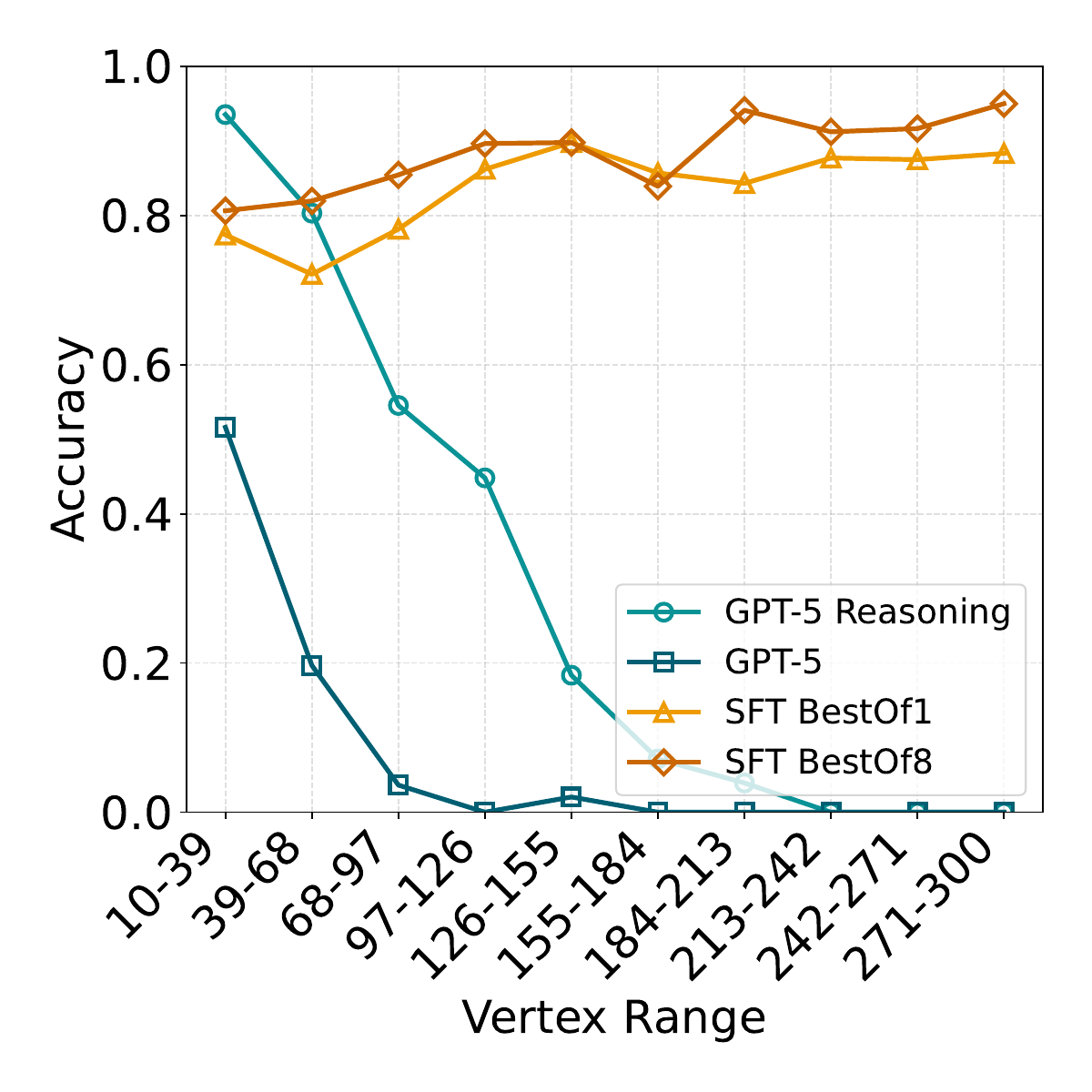}
        \caption{SAT}
        \label{fig:3-coloring-yes}
    \end{subfigure}
    \hfill
    \begin{subfigure}[b]{0.325\linewidth}
        \centering
        \includegraphics[width=\linewidth]{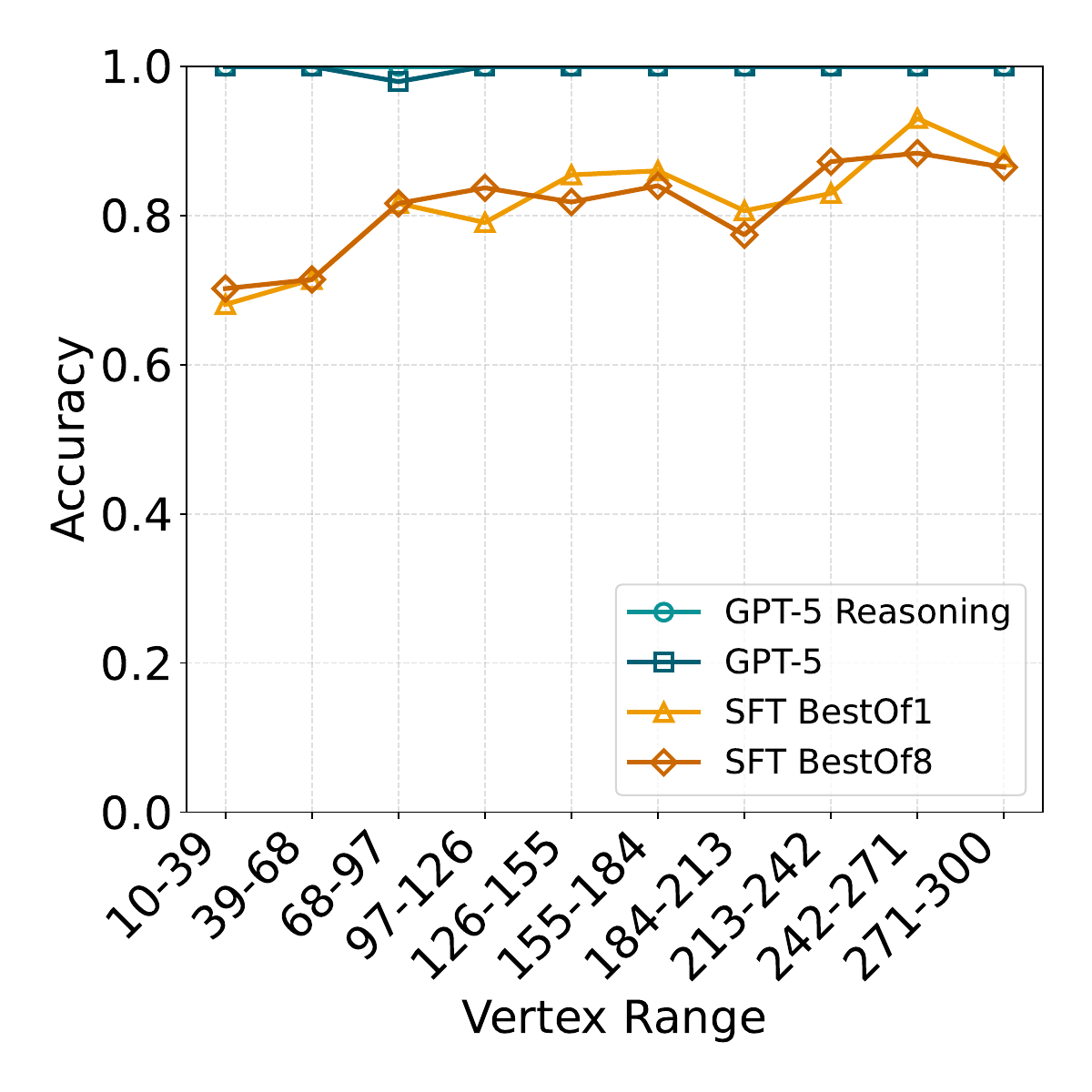}
        \caption{UNSAT}
        \label{fig:3-coloring-no}
    \end{subfigure}
    
    \caption{Accuracy breakdown by instance size. Subplots show results for (a) the full dataset, (b) SAT instances only, and (c) UNSAT instances only.}
    \label{fig:3-coloring-accuracy}
\end{figure}

Table~\ref{table:3-coloring} and Figure~\ref{fig:3-coloring-accuracy} report the corresponding results on 3-coloring. The same qualitative pattern appears. GPT-5.2 and GPT-5.2-Reasoning again achieve near-perfect performance on UNSAT instances, but their SAT accuracy collapses to 5.89\% and 28.33\%, respectively. This produces high apparent confidence on infeasibility detection but poor balanced performance overall.

In contrast, our fine-tuned model remains much more symmetric across the two classes. Best-of-1 already reaches 83.1\% infeasibility detection, with 83.84\% SAT accuracy and 82.28\% UNSAT accuracy, while Best-of-8 improves overall infeasibility detection to 89.2\% and raises SAT accuracy to 95.82\%. Although UNSAT accuracy is slightly lower than that of GPT-5.2-Reasoning, the resulting performance is far more balanced and therefore much more useful in practice. Figure~\ref{fig:3-coloring-accuracy} shows that this advantage is especially pronounced as the instance size grows: GPT-5.2 increasingly defaults to negative answers on larger SAT instances, while our fine-tuned model remains comparatively stable.

\subsection{Warm-Starting feasibility jump}
Next, we evaluate whether LLM outputs can serve as effective initializations for downstream local search. We use Feasibility Jump with a 30-seconds time limit and evaluate only on SAT instances. If a model predicts UNSAT or fails to return a valid structured output, we fall back to a fixed zero initialization. Since GPT-5.2 and GPT-5.2-Reasoning are the strongest non-fine-tuned baselines in infeasibility detection, we restrict the warm-start comparison to these models. To ensure consistency, we perform 5 trials for each method and instance.

\paragraph{Minor-embedding.}
\begin{table}[htbp]
\small
\begin{center}
\begin{tabular}{cccccc}

\toprule
 Model             & SAT Sol. Feas.               & Avg. Error                & FJ Succ.               & FJ time (s) & Gap$\downarrow$\\
\midrule
Random             & N/A            &  N/A           &   87.16\%           &  $5.05\pm 10.29 $     &     16.73\%           \\
GPT5.2             & 8.6\%           &   20.64        & 87.32\%          &     $4.99\pm 10.20 $     &    18.78\%  \\
GPT5.2-Reasoning   & \textbf{31.2\%}         &     14.33         & 87.76\%          & $4.86\pm10.12$  &  15.10\% \\
Llama-8b SFT BestOf1 & 11.2\%          &   15.20        & 87.32\%    &    $4.97\pm 10.25 $        &  16.33\%   \\
Llama-8b SFT BestOf8 & 23.4\%         &  \textbf{13.58}          & \textbf{91.40\%}         &  \textbf{$4.06\pm 9.20$}   &  \textbf{10.00\%}\\
\bottomrule
\end{tabular}
\end{center}
\caption{Feasibility Jump results for minor-embedding. SAT Sol. Feas. is the LLM's solution feasibility rate specifically for the satisfiable (SAT) instances, while Avg. Error measures intial solution's constraint violations. These metrics are N/A for the Random baseline as it is not LLM warm-started. FJ Succ. and FJ time denote success rate and runtime within 30s; Gap represents the relative objective gap ($(\text{method obj} - \text{best known obj}) / \text{method obj}$). }
\label{table:embedding-FJ}
\end{table}

Table~\ref{table:embedding-FJ} shows that better LLM-generated starting points translate into better downstream FJ performance. Among all methods, \emph{Llama-8b SFT BestOf8} produces the lowest average initialization error (13.58), the highest FJ success rate (91.40\%), the lowest final optimality gap (10.00\%), and the fastest convergence time ($4.06 \pm 9.20$ seconds). \emph{GPT-5.2-Reasoning} achieves the highest rate of directly feasible SAT outputs before repair (31.2\%), but its downstream optimization performance remains weaker than the SFT model once initialization quality and final gap are taken into account. 

\paragraph{3-coloring.}
\begin{table}[htbp]
\small
\begin{center}
\begin{tabular}{ccccc}

\toprule
 Model             & SAT Sol. Feas.               & Avg. Error                & FJ Succ.               & FJ time (s)\\
\midrule
Random             & N/A                        &  N/A                         &   94.98\%                       &  2.12$\pm$ 6.83                      \\
GPT5.2             & 4.37\%                 &    30.14                 & 94.71\%                  &     2.21$\pm$ 7.01             \\
GPT5.2-Reasoning   & \textbf{24.33\%}                 &  12.38                  & 95.25\%                  & 2.08$\pm$ 6.70\\
Llama-8b SFT BestOf1 & 2.28\%                 &  15.51                  & 94.80\%                  &    2.13$\pm$ 6.87               \\
Llama-8b SFT BestOf8 & \textbf{6.84\%}                 &  \textbf{12.30}                  & \textbf{98.56\%}                    &  \textbf{0.96$\pm$ 3.98}    \\
\bottomrule
\end{tabular}
\end{center}
\caption{Feasibility Jump Results of 3-coloring}
\label{table:3-coloring-FJ}
\end{table}

The same pattern appears in Table~\ref{table:3-coloring-FJ}. Although direct feasible outputs remain rare for all models, \emph{Llama-8b SFT BestOf8} produces the best starting points overall, with the lowest average initialization error (12.30), the highest FJ success rate (98.56\%), and the fastest convergence time ($0.96 \pm 3.98$ seconds). This supports the broader claim of the paper: even when the model is not fully correct on its own, its structured predictions can still provide useful guidance for conventional optimization algorithms.

\subsection{Out-of-distribution performance}

\begin{table}[htbp]
\small
\begin{center}
\begin{tabular}{cccc}

\toprule
Problem            & Inf. Det.               & SAT Acc.               & UNSAT Acc. \\
\midrule
3-coloring             & 88.00\%                        &  94.40\%                       &   81.60\%                                   \\
minor-embedding          &  87.70\%                  &       99.8\%             &        75.6\%                  \\
\bottomrule
\end{tabular}
\end{center}
\caption{Out of distribution performance.}
\label{table:OOD}
\end{table}

Table~\ref{table:OOD} reports out-of-distribution (OOD) generalization. Our model maintains strong performance on both problems, achieving 88.0\% infeasibility detection on 3-coloring and 87.7\% on minor-embedding. The OOD results are especially strong on SAT instances, where accuracy remains at 94.4\% for 3-coloring and 99.8\% for minor-embedding. Although UNSAT accuracy drops relative to the in-distribution setting, these results still indicate that the learned infeasibility-aware procedure transfers beyond the training distribution.

\subsection{Comparison with MILP solver}
We compare our framework with model-then-solve approaches, a common paradigm for applying LLMs to optimization~\citep{huang2025orlm,jiang2024llmopt,optimus03_2024}. In this baseline, the LLM generates both the optimization model and the solver code; details are given in Appendix~\ref{appendix:model-then-solve}. We test two Gurobi settings: default mode and the NoRel heuristic (gurobi's version of FJ). Default mode can find feasible solutions on SAT instances and certify infeasibility on UNSAT instances, whereas NoRel focuses on quickly finding feasible solutions and cannot certify infeasibility. For SAT instances, we report solution finding success rate on SAT instances, average relative gap to the best known objective ($(\text{method obj} - \text{best known obj}) / \text{method obj}$), and runtime. For UNSAT instances, we report the rate of correctly certifying infeasibility when applicable. Gurobi default mode is run for 60 seconds, and NoRel for 30 seconds.

\begin{table}[htbp]
\small
\begin{center}
\begin{tabular}{ccccc}

\toprule
    Method           & SAT Succ.              & Gap$\downarrow$               &UNSAT Succ.  & Time (s)\\
\midrule
Gurobi            & 54.6\%                        &  6.23\%                       &   17.2\%            & 47.02                       \\
Gurobi NoRel Heuristic          &  72.8\%                  &       5.76\%             &        N/A        &    25.32      \\
Llama-8b SFT BestOf8 + FJ& 93.8 \%                        &  8.02\%                       & 99.8\% & 10.22\\
\bottomrule
\end{tabular}
\end{center}
\caption{Comparison with Gurobi.}
\label{table:gurobi}
\end{table}
Table~\ref{table:gurobi} shows that our fine-tuned model achieves the highest SAT success rate and by far the strongest UNSAT performance, while also requiring the least runtime. Although Gurobi attains slightly smaller objective gaps on the SAT instances it solves, its coverage is much lower.

\section{Conclusion}
This paper proposes an infeasibility-aware framework for LLM-based combinatorial optimization that jointly addresses infeasibility detection and LLM-guided warm starts. For minor-embedding, we introduce a new mathematical programming formulation with zero-phase infeasibility screening, enabling the scalable generation of feasible and certifiably infeasible training instances. Our fine-tuned 8B model achieves 99.9\% overall accuracy on minor-embedding, far exceeding GPT-5.2 (78.5\%), and its warm-started pipeline reaches 91.4\% Feasibility Jump success and compares favorably with MILP solver baselines. We further show that the framework transfers beyond a single problem setting through strong results on graph coloring. Future work includes extending the framework to additional optimization problems, constructing richer reasoning trajectories for infeasibility detection, and improving scalability on long structured inputs through better representations and tighter integration with downstream optimization algorithms.
\clearpage



\bibliography{colm2026_conference}

@article{barabasi1999emergence,
  title={Emergence of scaling in random networks},
  author={Barab{\'a}si, Albert-L{\'a}szl{\'o} and Albert, R{\'e}ka},
  journal={science},
  volume={286},
  number={5439},
  pages={509--512},
  year={1999},
  publisher={American Association for the Advancement of Science}
}

@article{erdds1959random,
  title={On random graphs I},
  author={ERDdS, P and R\&wi, A},
  journal={Publ. math. debrecen},
  volume={6},
  number={290-297},
  pages={18},
  year={1959}
}

@article{watts1998collective,
  title={Collective dynamics of ‘small-world’networks},
  author={Watts, Duncan J and Strogatz, Steven H},
  journal={nature},
  volume={393},
  number={6684},
  pages={440--442},
  year={1998},
  publisher={Nature Publishing Group}
}

@techreport{hagberg2007exploring,
  title={Exploring network structure, dynamics, and function using NetworkX},
  author={Hagberg, Aric and Swart, Pieter J and Schult, Daniel A},
  year={2007},
  institution={Los Alamos National Laboratory (LANL)}
}

@article{hu2022lora,
  title={Lora: Low-rank adaptation of large language models.},
  author={Hu, Edward J and Shen, Yelong and Wallis, Phillip and Allen-Zhu, Zeyuan and Li, Yuanzhi and Wang, Shean and Wang, Liang and Chen, Weizhu and others},
  journal={Iclr},
  volume={1},
  number={2},
  pages={3},
  year={2022}
}

@article{wang2022self,
  title={Self-consistency improves chain of thought reasoning in language models},
  author={Wang, Xuezhi and Wei, Jason and Schuurmans, Dale and Le, Quoc and Chi, Ed and Narang, Sharan and Chowdhery, Aakanksha and Zhou, Denny},
  journal={arXiv preprint arXiv:2203.11171},
  year={2022}
}

@inproceedings{yao2022react,
  title={React: Synergizing reasoning and acting in language models},
  author={Yao, Shunyu and Zhao, Jeffrey and Yu, Dian and Du, Nan and Shafran, Izhak and Narasimhan, Karthik R and Cao, Yuan},
  booktitle={The eleventh international conference on learning representations},
  year={2022}
}

@inproceedings{huang2022language,
  title={Language models as zero-shot planners: Extracting actionable knowledge for embodied agents},
  author={Huang, Wenlong and Abbeel, Pieter and Pathak, Deepak and Mordatch, Igor},
  booktitle={International conference on machine learning},
  pages={9118--9147},
  year={2022},
  organization={PMLR}
}

@article{luteberget2023feasibility,
  title={Feasibility Jump: an LP-free Lagrangian MIP heuristic},
  author={Luteberget, Bj{\o}rnar and Sartor, Giorgio},
  journal={Mathematical Programming Computation},
  volume={15},
  number={2},
  pages={365--388},
  year={2023},
  publisher={Springer}
}

@misc{gurobi,
  author = {{Gurobi Optimization, LLC}},
  title = {{Gurobi Optimizer Reference Manual}},
  year = 2026,
  url = "https://www.gurobi.com"
}

@misc{ortools,
  title = {OR-Tools},
  version = { v9.12 },
  author = {Laurent Perron and Vincent Furnon},
  organization = {Google},
  url = {https://developers.google.com/optimization/},
  year = {2025}
}

@book{wolsey2014integer,
  title={Integer and combinatorial optimization},
  author={Wolsey, Laurence A and Nemhauser, George L},
  year={2014},
  publisher={John Wiley \& Sons}
}

@article{danna2005exploring,
  title={Exploring relaxation induced neighborhoods to improve MIP solutions},
  author={Danna, Emilie and Rothberg, Edward and Pape, Claude Le},
  journal={Mathematical Programming},
  volume={102},
  number={1},
  pages={71--90},
  year={2005},
  publisher={Springer}
}

@incollection{lodi2009mixed,
  title={Mixed integer programming computation},
  author={Lodi, Andrea},
  booktitle={50 Years of Integer Programming 1958-2008: From the Early Years to the State-of-the-Art},
  pages={619--645},
  year={2009},
  publisher={Springer}
}

@article{fischetti2003local,
  title={Local branching},
  author={Fischetti, Matteo and Lodi, Andrea},
  journal={Mathematical programming},
  volume={98},
  number={1},
  pages={23--47},
  year={2003},
  publisher={Springer}
}

@article{berthold2013measuring,
  title={Measuring the impact of primal heuristics},
  author={Berthold, Timo},
  journal={Operations Research Letters},
  volume={41},
  number={6},
  pages={611--614},
  year={2013},
  publisher={Elsevier}
}

@article{fischetti2005feasibility,
  title={The feasibility pump},
  author={Fischetti, Matteo and Glover, Fred and Lodi, Andrea},
  journal={Mathematical Programming},
  volume={104},
  number={1},
  pages={91--104},
  year={2005},
  publisher={Springer}
}

@article{da2025large,
  title={Large language models for combinatorial optimization: A systematic review},
  author={Da Ros, Francesca and Soprano, Michael and Di Gaspero, Luca and Roitero, Kevin},
  journal={ACM Computing Surveys},
  year={2025},
  publisher={ACM New York, NY}
}

@article{bengio2021machine,
  title={Machine learning for combinatorial optimization: a methodological tour d’horizon},
  author={Bengio, Yoshua and Lodi, Andrea and Prouvost, Antoine},
  journal={European Journal of Operational Research},
  volume={290},
  number={2},
  pages={405--421},
  year={2021},
  publisher={Elsevier}
}

@inproceedings{yang2023large,
  title={Large language models as optimizers},
  author={Yang, Chengrun and Wang, Xuezhi and Lu, Yifeng and Liu, Hanxiao and Le, Quoc V and Zhou, Denny and Chen, Xinyun},
  booktitle={The Twelfth International Conference on Learning Representations},
  year={2023}
}

@article{jiang2025large,
  title={Large language models as end-to-end combinatorial optimization solvers},
  author={Jiang, Xia and Wu, Yaoxin and Li, Minshuo and Cao, Zhiguang and Zhang, Yingqian},
  journal={arXiv preprint arXiv:2509.16865},
  year={2025}
}

@inproceedings{liu2024evolutionary,
  title={Large language models as evolutionary optimizers},
  author={Liu, Shengcai and Chen, Caishun and Qu, Xinghua and Tang, Ke and Ong, Yew-Soon},
  booktitle={2024 IEEE Congress on Evolutionary Computation (CEC)},
  pages={1--8},
  year={2024},
  organization={IEEE}
}

@inproceedings{xiao2023chain,
  title={Chain-of-experts: When llms meet complex operations research problems},
  author={Xiao, Ziyang and Zhang, Dongxiang and Wu, Yangjun and Xu, Lilin and Wang, Yuan Jessica and Han, Xiongwei and Fu, Xiaojin and Zhong, Tao and Zeng, Jia and Song, Mingli and others},
  booktitle={The twelfth international conference on learning representations},
  year={2023}
}

@article{jiang2024llmopt,
  title={LLMOPT: Learning to Define and Solve General Optimization Problems from Scratch},
  author={Jiang, Caigao and Shu, Xiang and Qian, Hong and Lu, Xingyu and Zhou, Jun and Zhou, Aimin and Yu, Yang},
  journal={arXiv preprint arXiv:2410.13213},
  year={2024}
}

@article{huang2025orlm,
  title={ORLM: A customizable framework in training large models for automated optimization modeling},
  author={Huang, Chenyu and Tang, Zhengyang and Hu, Shixi and Jiang, Ruoqing and Zheng, Xin and Ge, Dongdong and Wang, Benyou and Wang, Zizhuo},
  journal={Operations Research},
  year={2025},
  publisher={INFORMS}
}

@article{optimus2023,
  title         = {{OptiMUS}: Optimization Modeling Using {MIP} Solvers and Large Language Models},
  author        = {AhmadiTeshnizi, Ali and Gao, Wenzhi and Udell, Madeleine},
  journal       = {arXiv preprint arXiv:2310.06116},
  year          = {2023},
  eprint        = {2310.06116},
  archivePrefix = {arXiv},
  primaryClass  = {cs.AI},
  doi           = {10.48550/arXiv.2310.06116},
  url           = {https://arxiv.org/abs/2310.06116}
}

@article{optimus03_2024,
  title         = {{OptiMUS-0.3}: Using Large Language Models to Model and Solve Optimization Problems at Scale},
  author        = {AhmadiTeshnizi, Ali and Gao, Wenzhi and Brunborg, Herman and Talaei, Shayan and Lawless, Connor and Udell, Madeleine},
  journal       = {arXiv preprint arXiv:2407.19633},
  year          = {2024},
  eprint        = {2407.19633},
  archivePrefix = {arXiv},
  primaryClass  = {cs.AI},
  doi           = {10.48550/arXiv.2407.19633},
  url           = {https://arxiv.org/abs/2407.19633}
}

@article{ye2024reevo,
  title={ReEvo: Large language models as hyper-heuristics with reflective evolution},
  author={Ye, Haoran and Wang, Jiarui and Cao, Zhiguang and Berto, Federico and Hua, Chuanbo and Kim, Haeyeon and Park, Jinkyoo and Song, Guojie},
  journal={Advances in neural information processing systems},
  volume={37},
  pages={43571--43608},
  year={2024}
}

@article{funsearch2024,
  title   = {Mathematical discoveries from program search with large language models},
  author  = {Romera-Paredes, Bernardino and Barekatain, Mohammadamin and Novikov, Alexander and Balog, Matej and Kumar, M. Pawan and Dupont, Emilien and Ruiz, Francisco J. R. and Ellenberg, Jordan S. and Wang, Pengming and Fawzi, Omar and Kohli, Pushmeet and Fawzi, Alhussein},
  journal = {Nature},
  volume  = {625},
  number  = {7995},
  pages   = {468--475},
  year    = {2024},
  doi     = {10.1038/s41586-023-06924-6},
  url     = {https://doi.org/10.1038/s41586-023-06924-6}
}

@article{surina2025algorithm,
  title={Algorithm discovery with llms: Evolutionary search meets reinforcement learning},
  author={Surina, Anja and Mansouri, Amin and Quaedvlieg, Lars and Seddas, Amal and Viazovska, Maryna and Abbe, Emmanuel and Gulcehre, Caglar},
  journal={arXiv preprint arXiv:2504.05108},
  year={2025}
}

@article{ye2023satlm,
  title={Satlm: Satisfiability-aided language models using declarative prompting},
  author={Ye, Xi and Chen, Qiaochu and Dillig, Isil and Durrett, Greg},
  journal={Advances in Neural Information Processing Systems},
  volume={36},
  pages={45548--45580},
  year={2023}
}

@article{zhang2024dila,
  title={DiLA: Enhancing LLM tool learning with differential logic layer},
  author={Zhang, Yu and Zhen, Hui-Ling and Pei, Zehua and Lian, Yingzhao and Yin, Lihao and Yuan, Mingxuan and Yu, Bei},
  journal={arXiv preprint arXiv:2402.11903},
  year={2024}
}

@article{pan2024can,
  title={Can transformers reason logically? a study in sat solving},
  author={Pan, Leyan and Ganesh, Vijay and Abernethy, Jacob and Esposo, Chris and Lee, Wenke},
  journal={arXiv preprint arXiv:2410.07432},
  year={2024}
}

@article{hazra2025have,
  title={Have large language models learned to reason? a characterization via 3-sat phase transition},
  author={Hazra, Rishi and Venturato, Gabriele and Martires, Pedro Zuidberg Dos and De Raedt, Luc},
  journal={arXiv preprint arXiv:2504.03930},
  year={2025}
}

@inproceedings{wei2025satbench,
  title={Satbench: Benchmarking llms’ logical reasoning via automated puzzle generation from sat formulas},
  author={Wei, Anjiang and Wu, Yuheng and Wan, Yingjia and Suresh, Tarun and Tan, Huanmi and Zhou, Zhanke and Koyejo, Sanmi and Wang, Ke and Aiken, Alex},
  booktitle={Proceedings of the 2025 Conference on Empirical Methods in Natural Language Processing},
  pages={33820--33837},
  year={2025}
}

@article{chen2024diagnosing,
  title={Diagnosing infeasible optimization problems using large language models},
  author={Chen, Hao and Constante-Flores, Gonzalo E and Li, Can},
  journal={INFOR: Information Systems and Operational Research},
  volume={62},
  number={4},
  pages={573--587},
  year={2024},
  publisher={Taylor \& Francis}
}

@inproceedings{bernal2020integer,
  title={Integer programming techniques for minor-embedding in quantum annealers},
  author={Bernal, David E and Booth, Kyle EC and Dridi, Raouf and Alghassi, Hedayat and Tayur, Sridhar and Venturelli, Davide},
  booktitle={International Conference on Integration of Constraint Programming, Artificial Intelligence, and Operations Research},
  pages={112--129},
  year={2020},
  organization={Springer}
}

@article{kadowaki1998quantum,
  title={Quantum annealing in the transverse {I}sing model},
  author={Kadowaki, Tadashi and Nishimori, Hidetoshi},
  journal={Physical Review E},
  volume={58},
  number={5},
  pages={5355--5363},
  year={1998},
  publisher={APS}
}

@article{choi2008minor,
  title={Minor-embedding in adiabatic quantum computation: {I}. {T}he parameter setting problem},
  author={Choi, Vicky},
  journal={Quantum Information Processing},
  volume={7},
  number={5},
  pages={193--209},
  year={2008},
  publisher={Springer}
}

@article{choi2011minor,
  title={Minor-embedding in adiabatic quantum computation: {II}. {M}inor-universal graph design},
  author={Choi, Vicky},
  journal={Quantum Information Processing},
  volume={10},
  number={3},
  pages={343--353},
  year={2011},
  publisher={Springer}
}

@article{lobe2021minor,
  title={Minor embedding in broken {C}himera and {P}egasus graphs is {NP}-complete},
  author={Lobe, Elisabeth and Lutz, Annette},
  journal={arXiv preprint arXiv:2110.08325},
  year={2021}
}

@article{fang2020minimizing,
  title={Minimizing minor embedding energy: an application in quantum annealing},
  author={Fang, Yan-Long and Warburton, P A},
  journal={Quantum Information Processing},
  volume={19},
  pages={191},
  year={2020},
  publisher={Springer}
}

@article{cai2014practical,
  title={A practical heuristic for finding graph minors},
  author={Cai, Jun and Macready, William G and Roy, Aidan},
  journal={arXiv preprint arXiv:1406.2741},
  year={2014}
}

@article{boothby2016fast,
  title={Fast clique minor generation in {C}himera qubit connectivity graphs},
  author={Boothby, Tomas and King, Andrew D and Roy, Aidan},
  journal={Quantum Information Processing},
  volume={15},
  number={1},
  pages={495--508},
  year={2016},
  publisher={Springer}
}

@inproceedings{zbinden2020embedding,
  title={Embedding algorithms for quantum annealers with {C}himera and {P}egasus connection topologies},
  author={Zbinden, Simon and B{\"a}rtschi, Andreas and Djidjev, Hristo and Eidenbenz, Stephan},
  booktitle={International Conference on High Performance Computing},
  pages={187--206},
  year={2020},
  organization={Springer}
}

@article{sugie2021minor,
  title={Minor-embedding heuristics for large-scale annealing processors with sparse hardware graphs of up to 102,400 nodes},
  author={Sugie, Yuki and Yoshida, Yuki and Mertig, Normann and Takemoto, Takashi and Teramoto, Hiroshi and Nakamura, Atsushi and Takigawa, Ichigaku and Minato, Shin-ichi and Yamaoka, Masanao and Komatsuzaki, Tamiki},
  journal={Soft Computing},
  volume={25},
  number={3},
  pages={1731--1749},
  year={2021},
  publisher={Springer}
}

@article{dridi2018algebraic,
  title={A novel algebraic geometry compiling framework for adiabatic quantum computations},
  author={Dridi, Raouf and Alghassi, Hedayat and Tayur, Sridhar},
  journal={arXiv preprint arXiv:1810.01440},
  year={2018}
}

@article{nembrini2025minor,
  title={Minor Embedding for Quantum Annealing with Reinforcement Learning},
  author={Nembrini, Riccardo and Dacrema, Maurizio Ferrari and Cremonesi, Paolo},
  journal={arXiv preprint arXiv:2507.16004},
  year={2025}
}

@Inbook{karp2009reducibility,
  title={Reducibility Among Combinatorial Problems},
  author={Karp, Richard M},
  booktitle={50 Years of Integer Programming 1958-2008: from the Early Years to the State-of-the-Art},
  pages={219--241},
  year={2010},
  publisher={Springer}
}

@article{llama31herd,
      title={The Llama 3 Herd of Models}, 
      author={Grattafiori, Aaron and Dubey, Abhimanyu and Jauhri, Abhinav and Pandey, Abhinav and Kadian, Abhishek and Al-Dahle, Ahmad and Letman, Aiesha and Mathur, Akhil and Schelten, Alan and Vaughan, Alex and others},
      journal={arXiv preprint arXiv:2407.21783},
      year={2024}
}

@misc{unsloth,
  author = {Daniel Han, Michael Han and Unsloth team},
  title = {Unsloth},
  url = {https://github.com/unslothai/unsloth},
  year = {2023}
}

@misc{dwave_networkx,
    author       = {{D-Wave Systems Inc.}},
    title        = {dwave-networkx: A NetworkX extension providing graphs and algorithms relevant to working with the D-
  Wave System},
    year         = {2025},
    howpublished = {\url{https://github.com/dwavesystems/dwave-networkx}},
    note         = {Version 0.8.18; accessed 2026-03-28}
  }
\bibliographystyle{colm2026_conference}

\appendix

\section{Background of Mixed-Integer Linear Programming and Primal Heuristic}
\label{appendix:milp_background}

\paragraph{Mixed-integer linear programming.}
Mixed-integer linear programming (MILP) is a standard modeling framework for decision problems with a linear objective and linear constraints, where some decision variables are additionally required to take integer values, often binary values in $\{0,1\}$~\citep{wolsey2014integer, lodi2009mixed}. A generic MILP can be written as
$$
\min \; c^\top x \qquad
\text{s.t. } Ax \le b,\quad \ell \le x \le u,\quad x_j \in \mathbb{Z}\;\; \forall j \in I,
$$
where the integrality constraints encode discrete decisions such as selection, assignment, activation, or ordering. This makes MILP a flexible framework for many combinatorial optimization problems.

Modern MILP solvers typically rely on linear programming (LP) based branch-and-bound (together with cutting planes, presolve, and other enhancements)~\citep{wolsey2014integer}. The basic idea is to first drop the integrality requirements and solve the resulting linear programming relaxation. If the LP solution already satisfies integrality, then it is an optimal solution of the MILP. Otherwise, the solver branches on a fractional variable, creating smaller subproblems that are explored in a search tree. During this process, the solver maintains the best feasible integer solution found so far, called the \emph{incumbent}, as well as a bound from the remaining unexplored nodes. The gap between these two quantities measures how far the current incumbent may still be from optimality.

\paragraph{Primal heuristics.}
While branch-and-bound is designed to certify optimality, finding a good feasible solution early is often just as important in practice. This is the role of \emph{primal heuristics}: procedures that attempt to construct feasible integer solutions quickly, without necessarily proving optimality~\citep{fischetti2005feasibility,berthold2013measuring}. In a branch-and-bound solver, primal heuristics are used to improve the incumbent, which can immediately strengthen pruning and accelerate the overall search.

Primal heuristics can be very simple or quite sophisticated. Basic examples include rounding or diving heuristics based on the LP relaxation, while more advanced examples include \emph{Feasibility Pump}~\citep{fischetti2005feasibility}, \emph{Local Branching}~\citep{fischetti2003local}, and \emph{Relaxation Induced Neighborhood Search (RINS)}~\citep{danna2005exploring}. Conceptually, these methods trade exactness for speed: they try to find a feasible solution that is good enough to guide the search, even if they cannot certify that the solution is globally optimal.

\section{Background of Quantum Annealing and Minor-Embedding}
Quantum annealing (QA) is a hardware-based approach for approximately solving combinatorial optimization problems on a quantum processing unit (QPU)~\citep{kadowaki1998quantum}. In the standard QA setting, the objective is written as an Ising model, or equivalently a quadratically unconstrained binary optimization (QUBO), over variables with pairwise interactions. This induces a logical problem graph $P=(V_P,E_P)$, where each vertex represents a logical variable and each edge represents a nonzero interaction that must be realized on the device. The QPU, however, does not provide arbitrary all-to-all connectivity. Instead, it implements a fixed sparse hardware graph $G=(V_G,E_G)$ whose vertices are physical qubits and whose edges are programmable couplers.

Two representative hardware topologies are shown in Figure~\ref{fig:dwave-topologies}: Chimera and Pegasus. Pegasus offers higher local connectivity than Chimera, but both remain much sparser than many logical problem graphs of interest. Consequently, the logical graph usually cannot be mapped to the hardware graph directly. One must first construct a \emph{minor-embedding} of $P$ into $G$ before the instance can be executed on the device~\citep{choi2008minor,choi2011minor}.

\begin{figure}[htbp]
\centering
\includegraphics[width=0.98\linewidth]{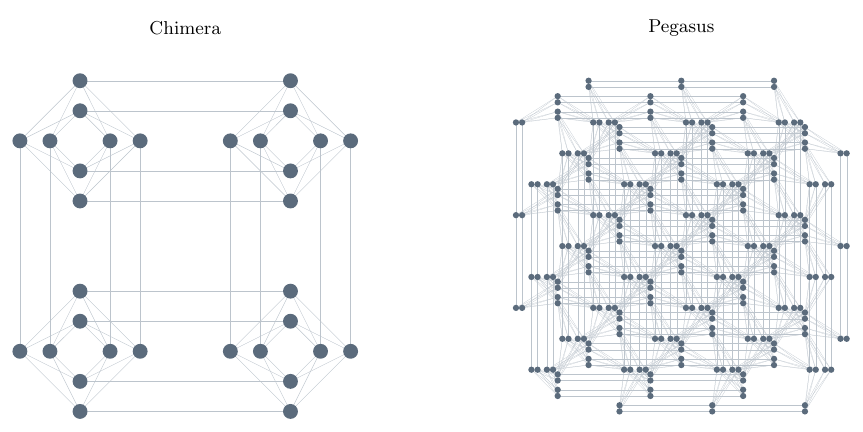}
\caption{Representative sparse hardware topologies used in quantum annealing: Chimera (left) and Pegasus (right). Pegasus has higher local connectivity than Chimera, but both remain sparse compared with many
  logical problem graphs. The figure was generated using the open-source \texttt{dwave\_networkx} package from exact Chimera and Pegasus graph constructions~\citep{dwave_networkx}; Chimera is shown with size parameter $(m=2, n=2, t=4)$ and
  Pegasus with $m=4$.}
\label{fig:dwave-topologies}
\end{figure}

In a minor-embedding, each problem vertex $i\in V_P$ is assigned to a nonempty connected set of hardware vertices, called a \emph{chain}. The qubits in a chain are coupled so that, during the anneal, they act as a single effective logical variable. Chains assigned to different problem vertices must be disjoint, and for every problem edge $\{i,k\}\in E_P$, there must be at least one hardware edge connecting the chain of $i$ to the chain of $k$. Thus, hardware couplers between chains realize the logical interactions required by the original problem.

This preprocessing step is essential because sparse hardware connectivity is an architectural constraint rather than a modeling choice. The embedding problem is also computationally hard in its own right~\citep{lobe2021minor}. In addition, embedding quality matters: longer chains use more hardware resources and are more prone to chain breaks during annealing, which can degrade solution quality~\citep{fang2020minimizing}. For this reason, the main body focuses not only on whether an embedding exists, but on whether one exists under an explicit chain-length limit.

\section{Additional formulation details for minor-embedding}
\label{appendix:minor-embedding}
\subsection{Implementation details of the chain-enumeration formulation}
\label{appendix:minor-embedding1}
We now provide additional implementation details for the chain-enumeration formulation. Although the formulation
\eqref{eq:objective-main}--\eqref{eq:feas-adj-main} is compact once the admissible chain family is fixed, the practical performance depends strongly on preprocessing, variable pruning, and the manner in which adjacency constraints~\eqref{eq:feas-adj-main} are enforced.

\paragraph{Chain enumeration and preprocessing.}
Given a hardware graph $G=(V_G,E_G)$ and a prescribed maximum chain size $L$, we first enumerate all connected hardware subgraphs $C \subseteq V_G$ with $|C| \le L$. For each enumerated chain $C \in \mathcal{C}$, we precompute its size $s(C)$, its set of internal hardware vertices, the number of internal hardware edges induced by $C$, and its external neighborhood
$$
\partial(C):=\{w \in V_G \setminus C : \exists\, u\in C \text{ with } \{u,w\}\in E_G\}.
$$
We also precompute the set of neighboring chains
$$
\Gamma(C):=\{D\in\mathcal{C} : \exists\, \{u,v\}\in E_G \text{ with } u\in C,\ v\in D\},
$$
which is later used to enforce adjacency compatibility between chains assigned to adjacent problem vertices.

\paragraph{Zero-phase infeasibility screening.}
Before constructing the optimization model, we apply Proposition~\ref{prop:zero-phase-main} as a zero-phase infeasibility screen. Let $\Delta_G:=\max_{v\in V_G}\deg_G(v)$, and for each $i\in V_P$ define
$$
s_i^{\min}:=
\max\!\left(1,\left\lceil\frac{\deg_P(i)-2}{\Delta_G-2}\right\rceil\right),
$$
which is the minimum chain size required by \eqref{eq:zp-degree-bound-main}. If $s_i^{\min}>L$ for any $i$, then no embedding with maximum chain size $L$ can exist. We then test the global vertex and edge bounds
$$
\sum_{i\in V_P} s_i^{\min} \leq |V_G|,
\qquad
|E_P|+\sum_{i\in V_P}(s_i^{\min}-1)\leq |E_G|.
$$
Violation of any of these three conditions certifies infeasibility, so the instance is discarded before chain enumeration or mixed-integer optimization.

\paragraph{Sparse variable generation.}
Rather than creating all variables $x_{iC}$ for $i\in V_P$ and $C\in\mathcal{C}$, we generate them selectively. A variable $x_{iC}$ is introduced only if the chain $C$ is large enough to support the problem vertex $i$, i.e.,
$$
s(C) \ge s_i^{\min},
$$
and if the chain has sufficient external neighbors to realize all incident problem edges, i.e.,
$$
|\partial(C)| \geq \deg_P(i).
$$
Equivalently, for each problem vertex $i$, we define the candidate chain set
$$
\mathcal{C}_i := \{C \in \mathcal{C} : s(C)\geq s_i^{\min},\ |\partial(C)|\geq \deg_P(i)\},
$$
and create variables only for pairs $(i,C)$ with $C\in\mathcal{C}_i$. If $\mathcal{C}_i=\emptyset$ for some $i$, the instance is certified infeasible. This filtering step can remove a substantial fraction of candidate variables and significantly reduces model size.

\paragraph{Edge-budget strengthening.}
In addition, we include a global edge-budget inequality. Let $e(C)$ denote the number of internal hardware edges induced by chain $C$. Since each selected chain of size $s(C)$ consumes internal hardware edges, and each problem edge requires at least one hardware edge between two selected chains, any feasible embedding must satisfy
$$
\sum_{i\in V_P}\sum_{C\in \mathcal{C}_i} e(C)x_{iC} + |E_P| \le |E_G|.
$$
This inequality is a valid strengthening cut that helps eliminate structurally impossible assignments early in the search.

\paragraph{Adjacency enforcement.}
The adjacency requirements in \eqref{eq:feas-adj-main} are enforced via \emph{lazy cuts} in Gurobi. We first solve the model with only \eqref{eq:feas-assign-main}--\eqref{eq:feas-disjoint-main}. For any incumbent integer solution, let $C_i^\star$ be the selected chain for each problem vertex $i$. For every problem edge $(i,k)\in A_P$, we check whether $C_i^\star$ and $C_k^\star$ are compatible, i.e., whether $C_k^\star \in \Gamma(C_i^\star)$. If not, then the adjacency condition for $(i,k)$ is violated, and we lazily add the missing constraint
$$
x_{i,C_i^\star} \le \sum_{D\in \Gamma(C_i^\star)} x_{kD}.
$$
Hence, adjacency constraints are generated only for chain selections that actually appear in incumbent solutions.

\subsection{Alternative formulation for minor-embedding}
\label{appendix:minor-embedding2}
As an alternative to enumerating admissible chains, one may model the minor-embedding problem directly on the hardware graph. In this approach, each problem vertex is assigned a subset of hardware vertices, and connectivity is enforced dynamically through feasibility cuts rather than by explicit inequalities. Following \citet{bernal2020integer}, this leads to a decomposition-based formulation in which the master problem enforces assignment, disjointness, and edge realization, while connectivity is checked separately.

\paragraph{Sets.}
Let $G=(V_G,E_G)$ denote the undirected hardware graph and let $P=(V_P,E_P)$ denote the problem graph.

\paragraph{Decision variables.}
For each $i\in V_P$ and $j\in V_G$, let $z_{ij}\in\{0,1\}$ indicate whether hardware vertex $j$ is assigned to the chain of problem vertex $i$. For each problem edge $\{i,k\}\in E_P$ and hardware edge $\{u,v\}\in E_G$, let $y_{ikuv}\in\{0,1\}$ indicate whether hardware edge $\{u,v\}$ is selected to realize problem edge $\{i,k\}$. Optionally, for each $i\in V_P$ and $\{u,v\}\in E_G$, let $f_{iuv}\in\{0,1\}$ indicate whether hardware edge $\{u,v\}$ is used internally within the chain of $i$.

\paragraph{Master problem.}
The master problem is
\begin{align}
\min\ & \sum_{i\in V_P}\sum_{j\in V_G} z_{ij} \label{eq:decomp-obj}\\
\text{s.t. } 
& \sum_{i\in V_P} z_{ij} \le 1
&& \forall j\in V_G, \label{eq:decomp-disjoint}\\
& \sum_{j\in V_G} z_{ij} \ge 1
&& \forall i\in V_P, \label{eq:decomp-nonempty}\\
& \sum_{\{u,v\}\in E_G} y_{ikuv} = 1
&& \forall \{i,k\}\in E_P, \label{eq:decomp-witness}\\
& y_{ikuv} \le z_{iu}+z_{iv}
&& \forall \{i,k\}\in E_P,\ \forall \{u,v\}\in E_G, \label{eq:decomp-y1}\\
& y_{ikuv} \le z_{ku}+z_{kv}
&& \forall \{i,k\}\in E_P,\ \forall \{u,v\}\in E_G, \label{eq:decomp-y2}\\
& z_{ij}\in\{0,1\}
&& \forall i\in V_P,\ \forall j\in V_G, \label{eq:decomp-zbin}\\
& y_{ikuv}\in\{0,1\}
&& \forall \{i,k\}\in E_P,\ \forall \{u,v\}\in E_G. \label{eq:decomp-ybin}
\end{align}
The objective \eqref{eq:decomp-obj} minimizes the total number of hardware vertices used. Constraints \eqref{eq:decomp-disjoint} enforce disjointness of chains, constraints \eqref{eq:decomp-nonempty} require each problem vertex to receive a nonempty chain, and constraints \eqref{eq:decomp-witness}--\eqref{eq:decomp-y2} ensure that each problem edge is realized by one hardware edge connecting the corresponding two chains.

\paragraph{Connectivity cuts.}
The master problem does not explicitly enforce that the hardware vertices assigned to a problem vertex induce a connected subgraph. Thus, connectivity is imposed dynamically through lazy cuts. For a current incumbent solution $z^*$ and a fixed $i\in V_P$, define
\begin{align}
\phi(i):=\{j\in V_G : z^*_{ij}=1\}.
\end{align}
If $\phi(i)$ is disconnected in the hardware graph, let $C$ be one connected component of $\phi(i)$ and let
\begin{align}
N(C):=\{w\in V_G\setminus C : \exists\, u\in C \text{ with } \{u,w\}\in E_G\}
\end{align}
denote the neighbors of $C$. Choose any $u\in C$ and any $v\in \phi(i)\setminus C$. Then any connected assignment containing both $u$ and $v$ must select at least one hardware vertex in $N(C)$, which yields the valid inequality
\begin{align}
z_{iu}+z_{iv}-1 \;\le\; \sum_{w\in N(C)} z_{iw}.
\label{eq:decomp-connectivity-cut}
\end{align}
At the current disconnected solution, both $u$ and $v$ are selected while no vertex in $N(C)$ is selected, so \eqref{eq:decomp-connectivity-cut} is violated. The master problem is re-solved with such cuts added until every assigned chain is connected.

\paragraph{Chain-size restrictions.}
We still impose an upper bound $L$ on the size of each chain. This yields the following constraints
\begin{align}
\sum_{j\in V_G} z_{ij} &\le L
&& \forall i\in V_P, \label{eq:decomp-size}\\
z_{iu}+z_{iv} &\le 1
&& \forall i\in V_P,\ \forall u,v\in V_G \text{ with } d_G(u,v)\ge L,
\label{eq:decomp-distance}
\end{align}
where $d_G(u,v)$ denotes the shortest-path distance in $G$. Indeed, if a connected subgraph contains at most $L$ vertices, then any two vertices in it must be at graph distance at most $L-1$.

\paragraph{Optional strengthening.}
The master problem can be strengthened by introducing edge-selection variables $f_{iuv}$ to capture internal edges of a chain. A simple necessary condition for connectivity is that a connected subgraph on $t$ vertices must contain at least $t-1$ edges. This yields
\begin{align}
f_{iuv} &\le z_{iu}
&& \forall i\in V_P,\ \forall \{u,v\}\in E_G, \label{eq:decomp-f1}\\
f_{iuv} &\le z_{iv}
&& \forall i\in V_P,\ \forall \{u,v\}\in E_G, \label{eq:decomp-f2}\\
\sum_{\{u,v\}\in E_G} f_{iuv} &\ge \sum_{j\in V_G} z_{ij}-1
&& \forall i\in V_P, \label{eq:decomp-f3}\\
f_{iuv} &\in \{0,1\}
&& \forall i\in V_P,\ \forall \{u,v\}\in E_G. \label{eq:decomp-fbin}
\end{align}
These inequalities do not by themselves guarantee connectivity, but they tighten the relaxation by ruling out some disconnected assignments.

\paragraph{Formulation size.}
The master problem contains $|V_P||V_G|$ binary variables $z_{ij}$ and $|E_P||E_G|$ binary variables $y_{ikuv}$, plus optionally $|V_P||E_G|$ binary variables $f_{iuv}$. The number of constraints is polynomial in $|V_P|$, $|E_P|$, $|V_G|$, and $|E_G|$, with connectivity enforced iteratively through the feasibility cuts \eqref{eq:decomp-connectivity-cut}.

\paragraph{Comparison to Chain-enumeration formulation}
We compare the two formulations on instances that are not eliminated by the zero-phase infeasibility screening. The results are reported in Table~\ref{table:minor-embedding-comparison}.

\begin{table}[htbp]
\small
\begin{center}
\begin{tabular}{ccccc}
\toprule
Formulation & Solved  & UNSAT  & Gap$\downarrow$ & Time (s)\\
\midrule
Decomposition       & 65.41\%      & 6.20\%   & 15.15\% & 42.95\\
Chain-enumeration   & 96.40\%      & 100.00\%   & 5.14\% & 17.88\\
\bottomrule
\end{tabular}
\end{center}
\caption{Comparison of the two formulations on 2000 instances with a 60-second Gurobi time limit per instance. The table reports results on the 1100 instances that were solved by at least one method or certified infeasible. \emph{Solved} is the fraction of instances solved within the time limit, \emph{UNSAT} is the fraction of infeasible instances certified infeasible, \emph{Gap} is the mean final relative optimality gap at termination (returned by Gurobi), computed as $(\text{upper bound}-\text{lower bound})/\lvert \text{upper bound}\rvert$, and \emph{Time (s)} is the mean total runtime.}
\label{table:minor-embedding-comparison}
\end{table}
Overall, the chain-enumeration formulation is more effective on this test set. It solves more instances within the time limit, detects infeasibility more reliably, and yields smaller final gaps. It also requires less average total runtime. These results support its use as the main formulation for dataset generation.

\section{Formulation Details for k-coloring}
\label{appendix:k-coloring}
\subsection{k-coloring}
Given an undirected graph $G=(V,E)$ and a fixed integer $k\ge 1$, the $k$-coloring feasibility problem asks whether the vertices of $G$ can be assigned at most $k$ colors so that adjacent vertices receive different colors.

\paragraph{Decision variables.}
For each vertex $i\in V$ and color $c\in K$, let $x_{ic}\in\{0,1\}$ indicate whether vertex $i$ is assigned color $c$.

\paragraph{Formulation.}
The $k$-coloring feasibility problem can be written as
\begin{align}
\sum_{c\in K} x_{ic} &= 1
&& \forall i\in V, \label{eq:kcol-assign}\\
x_{ic} + x_{jc} &\le 1
&& \forall \{i,j\}\in E,\ \forall c\in K, \label{eq:kcol-edge}
\end{align}
Constraints \eqref{eq:kcol-assign} ensure that each vertex receives exactly one color, while constraints \eqref{eq:kcol-edge} enforce that adjacent vertices cannot be assigned the same color. Thus, the graph is $k$-colorable if and only if the above system is feasible.

Note that the 2-coloring problem is equivalent to determining whether a graph is bipartite, which is solvable in linear time. So in this paper we consider the 3-coloring problem.
\subsection{min-coloring}
Given an undirected graph $G=(V,E)$, the min-coloring problem asks for a proper coloring of the vertices using as few colors as possible. Equivalently, the goal is to determine the chromatic number of $G$.

\paragraph{Candidate colors.}
Let $K:=\{1,2,\dots,|V|\}$ be an upper-bounding set of candidate colors. Since every graph on $|V|$ vertices can be colored with at most $|V|$ colors, this choice is always valid.

\paragraph{Decision variables.}
For each vertex $i\in V$ and color $c\in K$, let $x_{ic}\in\{0,1\}$ indicate whether vertex $i$ is assigned color $c$. In addition, for each color $c\in K$, let $y_c\in\{0,1\}$ indicate whether color $c$ is used by at least one vertex.

\paragraph{Formulation.}
The min-coloring problem can be written as
\begin{align}
\min \sum_{c\in K} y_c \label{eq:mincol-obj}
\end{align}
subject to
\begin{align}
\sum_{c\in K} x_{ic} &= 1
&& \forall i\in V, \label{eq:mincol-assign}\\
x_{ic} + x_{jc} &\le 1
&& \forall \{i,j\}\in E,\ \forall c\in K, \label{eq:mincol-edge}\\
x_{ic} &\le y_c
&& \forall i\in V,\ \forall c\in K, \label{eq:mincol-link}
\end{align}
Constraints \eqref{eq:mincol-assign} assign exactly one color to each vertex, constraints \eqref{eq:mincol-edge} ensure that adjacent vertices receive different colors, and constraints \eqref{eq:mincol-link} activate $y_c$ whenever color $c$ is assigned to some vertex. The objective \eqref{eq:mincol-obj} therefore minimizes the number of colors used.

\paragraph{Symmetry breaking.}
Because colors are interchangeable, the formulation contains substantial symmetry. A common strengthening is to impose
\begin{align}
y_c \ge y_{c+1}
\qquad \forall c=1,\dots,|V|-1, \label{eq:mincol-sym}
\end{align}
which enforces that lower-indexed colors are used before higher-indexed ones.

\subsection{Difficulty of infeasibility vs. optimality}
\label{appendix:min-coloring}
\begin{table}[htbp]
\small
\begin{center}
\begin{tabular}{cccc}
\toprule
Model              & Feasibility$\uparrow$ & Optimality Gap$\downarrow$ & Error Ratio$\downarrow$\\
\midrule
GPT5.2             & 22.2\%      & 70.58\%   & 17.04\%\\
Llama-8b SFT BestOf1 & 10.1\%      & 16.64\%   & 3.90\%        \\
Llama-8b SFT BestOf8 & 24.1\%      & 10.04\%   & 2.06\% \\
\bottomrule
\end{tabular}
\end{center}
\caption{Evaluation of feasibility and optimality gap (min-coloring). Bold indicates optimal performance per metric. Error Ratio is defined as the percentage of violated edges among all edges.}
\label{table:min-coloring}
\end{table}

Table~\ref{table:min-coloring} presents the results for the min-coloring problem. Unlike the 3-coloring task, this problem does not include ground-truth infeasible instances. Although the min-coloring problem shares constraints with 3-coloring and involves a larger decision space (e.g., determining the optimal number of colors), the model performs significantly better at generating valid colorings. (Compare to SAT instances feasibility rate in table~\ref{table:3-coloring-FJ}). While techniques like reinforcement learning, as proposed by~\citet{jiang2025large}, could further enhance the fine-tuned model's performance, we leave such enhancement for future work, as this study mainly focus on infeasibility detection.

\section{Proofs}
\label{appendix:proofs}
\subsection{Proof of Theorem \ref{thm:meec-feas-npc}}
\begin{proof}
Membership in NP is immediate. Given a binary assignment vector $x$, one can verify in polynomial time that each problem vertex is assigned exactly one chain, that the selected chains are pairwise disjoint, and that for every problem edge the corresponding selected chains are adjacent in the hardware graph.

To prove NP-hardness, we reduce from \emph{3-Dimensional Matching (3DM)}.
An instance of \emph{3DM} consists of three pairwise disjoint sets
$X,Y,Z$ with $|X|=|Y|=|Z|=n,$ and a set of triples $T\subseteq X\times Y\times Z.$
The question is whether there exists a subset $M\subseteq T$ with $|M|=n$ such that no two triples in $M$ share an element.

We now construct an instance of \emph{MEEC-Feas} from $(X,Y,Z,T)$.
\paragraph{Hardware graph.}
Given an instance $(X,Y,Z,T)$ of \emph{3DM}, we construct a hardware graph $G=(V_G,E_G)$ with three types of vertices.

First, for each part $B\in\{X,Y,Z\}$, create four special vertices
$$
a_B,\; b_B^1,\; b_B^2,\; b_B^3,
$$
and add all six edges among them. Thus, for each $B \in \{X,Y,Z\}$, these four vertices induce a complete graph on four vertices, denoted by $K_4$.

Second, create a corresponding hardware vertex for each element in $X \cup Y \cup Z$. Add edges from the $a$-vertices to these element-vertices: $\{a_X,x\}$ for all $x \in X$, $\{a_Y,y\}$ for all $y \in Y$, and $\{a_Z,z\}$ for all $z \in Z$.

Finally, for each triple $t=(x,y,z) \in T$, create one additional hardware vertex $p_t$, and connect it to the three element-vertices appearing in that triple: $\{p_t,x\}, \{p_t,y\}$, and $\{p_t,z\}$. There are no other hardware edges.

\paragraph{Admissible chain family.}
Let
$$
\mathcal C
:=
\bigl\{\{v\}:v\in V_G\bigr\}
\;\cup\;
\bigl\{C_t : t\in T\bigr\},
$$
where for each triple $t=(x,y,z) \in T$, we define $C_t := \{p_t,x,y,z\}$. Each $C_t$ is connected in $G$ (inducing a star centered at $p_t$), and by construction, every singleton chain $\{v\}$ is included in $\mathcal{C}$.

\paragraph{Problem graph.}
Construct a problem graph $P=(V_P,E_P)$ as follows.

For each part $B \in \{X,Y,Z\}$, create four vertices $u_B^0, u_B^1, u_B^2, u_B^3$ and add all six edges among them, ensuring that each set of four induces a $K_4$ in $P$. In addition, create a set of $n$ selector vertices, denoted as $S = \{s_r \mid 1 \le r \le n\}$. For every $s_r \in S$ and each part $B \in \{X,Y,Z\}$, add the edge $\{s_r,u_B^0\}$. There are no other edges in $P$.

We now show that the constructed \emph{MEEC-Feas} instance is feasible if and only if the original \emph{3DM} instance has a matching of size $n$. We rely on two key observations: First, each $K_4$ in the problem graph must be mapped to one of the three $K_4$ formed by singleton chains in the hardware graph. Second, once established, each selector vertex is forced to choose a chain of the form $C_t$.

\paragraph{Claim 1.}
The only collections of four pairwise disjoint admissible chains that are pairwise adjacent are 
$$\bigl\{\{a_B\},\{b_B^1\},\{b_B^2\},\{b_B^3\}\bigr\}$$ 
for $B \in \{X,Y,Z\}$.

\noindent
\emph{Proof of Claim 1.}
If four pairwise disjoint singleton chains are pairwise adjacent, their underlying hardware vertices must form a $K_4$ in $G$. By construction, the only $K_4$ subgraphs of $G$ are the three gadgets corresponding to $X, Y$, and $Z$.

It remains to rule out the possibility that one of the four chains is a nonsingleton chain $C_t$. Suppose $C_t=\{p_t,x,y,z\}$ is selected. Any admissible chain disjoint from $C_t$ and adjacent to $C_t$ must contain a vertex adjacent in $G$ to one of $p_t,x,y,z$.

Because $p_t$ connects exclusively to $x, y,$ and $z$, any singleton chain $\{p_t\}$ intersects $C_t$ and is not disjoint. Furthermore, singletons $\{x\}, \{y\},$ and $\{z\}$ also intersect $C_t$. The only disjoint singleton chains adjacent to $C_t$ are therefore $\{a_X\},\{a_Y\},\{a_Z\}$ (via the edges $\{a_X,x\}, \{a_Y,y\}, \{a_Z,z\}$), along with any singleton $\{p_{t'}\}$ where $t'$ shares an element with $t$. However, the vertices $a_X, a_Y,$ and $a_Z$ are pairwise non-adjacent in $G$, and all $p$-vertices are pairwise non-adjacent to each other. Consequently, it is impossible to find three mutually adjacent chains among them. Hence, $C_t$ cannot belong to any family of four pairwise disjoint, pairwise adjacent chains.
\hfill$\diamond$

Since each set $\{u_B^0,u_B^1,u_B^2,u_B^3\}$ induces a $K_4$ in $P$, Claim 1 implies that in any feasible solution these four problem vertices must be assigned to one of the three hardware $K_4$ gadgets. Because these gadgets are pairwise disjoint and chain disjointness is enforced globally, the three $K_4$ subgraphs of $P$ must occupy the three hardware gadgets bijectively.

\paragraph{Claim 2.}
Within each hardware $K_4$ gadget, the vertex $u_B^0$ must be assigned to the singleton chain $\{a_{B'}\}$ of that gadget (for some $B'\in\{X,Y,Z\}$), and not to one of the chains $\{b_{B'}^1\},\{b_{B'}^2\},\{b_{B'}^3\}$.

\noindent
\emph{Proof of Claim 2.}
The vertex $u_B^0$ is adjacent in $P$ not only to the other three vertices of its $K_4$ but also to all selector vertices in $S$. If $u_B^0$ were assigned to, say, $\{b_{B'}^1\}$, then every selector $s_r$ would have to be assigned to a chain adjacent to $\{b_{B'}^1\}$. But in $G$, the only vertices adjacent to $b_{B'}^1$ are the other three vertices of the same hardware $K_4$ gadget, namely $a_{B'},b_{B'}^2,b_{B'}^3$. Those singleton chains are already occupied by the four vertices of that problem $K_4$, and no nonsingleton chain in $\mathcal C$ contains any of these gadget vertices. Hence no selector vertices can be assigned a chain adjacent to $\{b_{B'}^1\}$, contradiction. The same argument applies to $\{b_{B'}^2\}$ and $\{b_{B'}^3\}$. Therefore, $u_B^0$ must be assigned to $\{a_{B'}\}$.
\hfill$\diamond$

Up to a permutation of the labels $X, Y,$ and $Z$, it follows that the three distinguished vertices $u_X^0, u_Y^0, u_Z^0$ are strictly assigned to $\{a_X\}, \{a_Y\}$, and $\{a_Z\}$.

Now consider any selector vertex $s_r$. Since $s_r$ is adjacent in $P$ to all three $u^0$ vertices, its assigned chain must be adjacent to all three singleton chains $\{a_X\},\{a_Y\},\{a_Z\}$. We evaluate the admissible chains to find which possess this property:
\vspace{-1em}
\begin{enumerate}[label=-]
    \item A singleton chain $\{b_B^j\}$ is adjacent only to the other three vertices in its own $K_4$, so it cannot be adjacent to all three of $\{a_X\}, \{a_Y\}, \{a_Z\}$.
    \vspace{-0.5em}
    \item A singleton chain $\{a_B\}$ is already used by one of the vertices $u_X^0,u_Y^0,u_Z^0$, and is therefore unavailable by disjointness.
    \vspace{-0.5em}
    \item A singleton chain $\{x\}$ with $x\in X$ is adjacent only to $\{a_X\}$ among these three singleton chains. Similarly, $\{y\}$ with $y\in Y$ is adjacent only to $\{a_Y\}$, and $\{z\}$ with $z\in Z$ is adjacent only to $\{a_Z\}$.
    \vspace{-0.5em}
    \item A singleton chain $\{p_t\}$ is adjacent only to the three element-vertices of the triple $t$, and so it is adjacent to none of $\{a_X\}, \{a_Y\}, \{a_Z\}$.
    \vspace{-0.5em}
    \item For a triple $t=(x,y,z)$, the chain $C_t=\{p_t,x,y,z\}$ is adjacent to all three singleton chains $\{a_X\}, \{a_Y\}, \{a_Z\}$ through the edges $\{a_X,x\}$, $\{a_Y,y\}$, and $\{a_Z,z\}$.
\end{enumerate}
\vspace{-1em}
Therefore, every selector vertex in $S$ must be assigned to a chain of the form $C_t$.

Since each $p_t$ is unique to its triple, two chains $C_t$ and $C_{t'}$ intersect if and only if the corresponding triples $t$ and $t'$ share an element of $X\cup Y\cup Z$. Hence, choosing $n$ pairwise disjoint selector chains is equivalent to choosing $n$ triples from $T$ that share no element, i.e., to a 3-dimensional matching of size $n$.

\paragraph{\emph{MEEC-Feas}$\Rightarrow$\emph{3DM}.}

If the constructed \emph{MEEC-Feas} instance is feasible, then the $n$
selector vertices determine $n$ pairwise disjoint chains $C_t$, and hence
$n$ triples in $T$ that are pairwise disjoint. Therefore the original
\emph{3DM} instance has a matching of size $n$.

\paragraph{\emph{MEEC-Feas}$\Leftarrow$\emph{3DM}.}
Conversely, suppose the \emph{3DM} instance has a matching $M = \{t_r \mid 1 \le r \le n\} \subseteq T$ of size $n$. Assign each of the three $K_4$ subgraphs of $P$ to one of the three hardware $K_4$ gadgets (e.g., $u_B^0 \mapsto \{a_B\}$ and $u_B^j \mapsto \{b_B^j\}$). Then, assign $s_r \mapsto C_{t_r}$ for all $1 \le r \le n$. Because the matching triples are pairwise disjoint, the selected chains $\{C_{t_r} \mid 1 \le r \le n\}$ are pairwise disjoint. Furthermore, each $C_{t_r}$ connects to the three singleton chains $\{a_X\}, \{a_Y\},$ and $\{a_Z\}$, thereby realizing all problem edges $\{s_r, u_B^0\}$ for $B \in \{X, Y, Z\}$. Thus all constraints \eqref{eq:feas-assign-main}--\eqref{eq:feas-adj-main} are satisfied.

The construction is polynomial in the size of the \emph{3DM} instance. Moreover, all admissible nonsingleton chains introduced by the reduction have cardinality $4$. Therefore, NP-hardness already holds in this restricted special case, and
therefore \emph{MEEC-Feas} is NP-complete.
\end{proof}

\subsection{Proof of the zero-phase infeasibility certificate}
For each problem vertex $i\in V_P$, let $C_i\subseteq V_H$ denote its chain, and let $s_i:=|C_i|$. Since $C_i$ is connected, the subgraph induced by $C_i$ contains at least $s_i-1$ internal hardware edges.
\begin{proof}
Fix any problem vertex $i\in V_P$ and let $s_i:=|C_i|$. Since each hardware vertex has degree at most $\Delta_H$, the total degree contributed by the vertices of $C_i$ is at most $s_i\Delta_H$. Because $C_i$ is connected, it contains at least $s_i-1$ internal hardware edges, and each such edge consumes two units of degree. Therefore, the number of hardware edges leaving $C_i$ is at most
$$
s_i\Delta_H-2(s_i-1)=s_i(\Delta_H-2)+2.
\label{eq:boundary-count-app}
$$
Each neighbor of $i$ in the problem graph must be realized by at least one hardware edge leaving $C_i$, since chains assigned to distinct problem vertices are disjoint and every problem edge must be witnessed by a hardware edge between the corresponding chains. Hence
$$
\deg_P(i)\le s_i(\Delta_H-2)+2 \le L(\Delta_H-2)+2,
$$
which proves \eqref{eq:zp-degree-bound-main}. Rearranging this inequality gives the lower bound on the chain size required for vertex $i$.

Because the chains are pairwise disjoint, any embedding uses at least $s_i^{\min}$ hardware vertices for each $i\in V_P$. Summing over all problem vertices yields
$$
\sum_{i\in V_P} s_i^{\min}\le |V_H|,
$$
which proves \eqref{eq:zp-node-bound-main}.

For \eqref{eq:zp-edge-bound-main}, note that each connected chain $C_i$ with $|C_i|=s_i$ contains at least $s_i-1$ internal hardware edges. In addition, each problem edge in $E_P$ requires at least one hardware edge between two distinct chains. Internal chain edges lie entirely within a single chain, whereas witness edges for problem edges connect two distinct chains; therefore these two classes of hardware edges are disjoint. It follows that any embedding must use at least
$$
|E_P|+\sum_{i\in V_P}(s_i-1)
$$
hardware edges, and hence at least
$$
|E_P|+\sum_{i\in V_P}(s_i^{\min}-1).
$$
Therefore,
$$
|E_P|+\sum_{i\in V_P}(s_i^{\min}-1)\le |E_H|,
$$
which proves \eqref{eq:zp-edge-bound-main}.
\end{proof}

\section{Details of feasibility jump}
\label{appendix:fj}
For completeness, we summarize the Feasibility Jump (FJ) repair procedure used in our framework. Let
$N := \{1,\dots,n\},
\quad
M := \{1,\dots,m\},$
denote the index sets of variables and constraints, respectively. We consider the linear feasibility problem
$$
\text{find } x \in \mathcal{X}
\quad \text{s.t.} \quad
a_i^\top x \le b_i \qquad \forall i \in M,
$$
where $\mathcal{X}$ encodes bounds and integrality restrictions. FJ works with the weighted infeasibility objective
$$
F_w(x)=\sum_{i\in M} w_i f_i(x),
\qquad
f_i(x)=\max\{0,\, a_i^\top x-b_i\},
$$
where $w_i \ge 0$ is the current weight of constraint $i$. Given a current assignment $\bar x$, FJ evaluates for each variable $x_j$ the best one-variable move with all other variables fixed. Let
$
G_j(t \mid \bar x_{k\neq j})
$
denote the weighted infeasibility obtained by setting $x_j=t$ while keeping all other variables fixed. The \emph{jump value} of variable $j$ is
$$
v_j
:=
\mathrm{Jump}_j(\bar x_{k\neq j})
=
 \mathrm{argmin}_{t \in [l_j,u_j]\cap D_j,\ t\neq \bar x_j}
G_j(t \mid \bar x_{k\neq j}),
$$
where $D_j=\mathbb{Z}$ for integer variables and $D_j=\mathbb{R}$ otherwise. Its score is
$$
s_j
=
G_j(\bar x_j \mid \bar x_{k\neq j})
-
G_j(v_j \mid \bar x_{k\neq j}),
$$
which measures the reduction in weighted infeasibility obtained by moving $x_j$ to its jump value. The set of promising variables is
$
P := \{j \in N : s_j > 0\}.
$
If $P=\varnothing$ and the current solution is still infeasible, FJ regards the current point as a local minimum. It then increases the weights of violated constraints, updates the affected scores, and continues the search from the reweighted objective.

The pseudocode is given in Algorithm~\ref{alg:feasibility_jump}.
\begin{algorithm}[t]
\caption{Feasibility Jump Repair from an LLM Candidate}
\label{alg:feasibility_jump}
\DontPrintSemicolon
\SetKwInOut{Input}{Input}
\SetKwInOut{Output}{Output}

\Input{LLM candidate $\hat{x}$, iteration limit $T$}
\Output{A feasible solution $x\in\mathcal X$ if found; otherwise failure}

Parse $\hat{x}$ into an initial assignment $x\in\mathcal X$\;
Initialize $w_i \gets 1$ for all $i\in M$\;
Compute $v_j$ and $s_j$ for all $j\in N$\;
$P \gets \{j\in N : s_j>0\}$\;
$x^{\mathrm{best}} \gets x$\;

\For{$t \gets 1$ \KwTo $T$}{
    \If{$F_w(x)=0$}{
        \Return{$x$}\;
    }

    \eIf{$P\neq \varnothing$}{
        Choose a subset $P^\star \subseteq P$ with at most $25$ variables\;
        Select
        $
        j^\star \in \arg\max_{j\in P^\star} s_j
        $
    }{
        $U \gets \{i\in M : f_i(x)>0\}$\;
        \ForEach{$i\in U$}{
            $w_i \gets w_i + 1$\;
        }
        Recompute $v_j$ and $s_j$ for all $j\in N$ with $a_{ij}\neq 0$ for some $i\in U$\;
        $P \gets \{j\in N : s_j>0\}$\;
        Choose $i^\star$ uniformly at random from $U$\;
        Select
        $
        j^\star \in \arg\max_{j\in N:\, a_{i^\star j}\neq 0} s_j
        $
    }

    $x_{j^\star} \gets v_{j^\star}$\;

    \If{$F_w(x) < F_w(x^{\mathrm{best}})$}{
        $x^{\mathrm{best}} \gets x$\;
    }

    Recompute $v_{j^\star}$ and $s_{j^\star}$\;
    Recompute $v_j$ and $s_j$ for all $j\neq j^\star$ sharing a constraint with $j^\star$\;
    $P \gets \{j\in N : s_j>0\}$\;
}

\If{$F_w(x^{\mathrm{best}})=0$}{
    \Return{$x^{\mathrm{best}}$}\;
}
\Return{failure}\;
\end{algorithm}

\paragraph{Phase two of feasibility jump.}
After a feasible solution has been found, we run a second local-improvement phase that minimizes the number of used hardware vertices while preserving feasibility. Let $\phi(x)$ denote the total number of hardware vertices used by the current embedding. Phase 2 considers only deletion moves: for each chain of size at least two, it proposes removing a small number of hardware vertices from that chain. A candidate move is accepted only if the updated solution remains feasible and strictly decreases $\phi(x)$. In our implementation, at most 60 candidate deletions are generated per iteration, at most 20 are probed, and the search terminates after a fixed number of consecutive non-improving iterations. The best feasible solution found during this phase is returned. The pseudocode is given in Algorithm~\ref{alg:fj-phase2}. In the experiments reported in Table~\ref{table:embedding-FJ}, Phase 2 is executed as long as the total runtime remains within the 30-second limit, with the patience parameter set to 40.

\begin{algorithm}[t]
\caption{Phase 2 of Feasibility Jump for minor-embedding}
\label{alg:fj-phase2}
\DontPrintSemicolon
\SetKwInOut{Input}{Input}
\SetKwInOut{Output}{Output}

\Input{A feasible assignment $x$, patience threshold $P_{\max}$}
\Output{A feasible assignment with improved objective, if found}

$\phi(x) \gets$ total number of used hardware vertices\;
$x^{\mathrm{best}} \gets x$\;
$\phi^{\mathrm{best}} \gets \phi(x)$\;
$p \gets 0$\;

\While{$p < P_{\max}$}{
    Construct candidate deletion moves $\mathcal{M}$ as follows:\;
    \Indp
    For each problem vertex, if its chain has size at least two, sample up to four hardware vertices from that chain and create moves that unassign those vertices\;
    Stop once $|\mathcal{M}|=60$\;
    \Indm

    \If{$\mathcal{M}=\varnothing$}{
        \Return{$x^{\mathrm{best}}$}\;
    }

    Sample a probe set $\mathcal{M}^\star \subseteq \mathcal{M}$ with $|\mathcal{M}^\star| \le 20$\;
    improved $\gets$ \textbf{false}\;

    \ForEach{$m \in \mathcal{M}^\star$}{
        Apply $m$ to obtain trial solution $x'$ \;
        Re-evaluate all feasibility constraints for $x'$\;
        \If{$x'$ is feasible \textbf{and} $\phi(x') < \phi(x)$}{
            $x \gets x'$\;
            \If{$\phi(x) < \phi^{\mathrm{best}}$}{
                $x^{\mathrm{best}} \gets x$\;
                $\phi^{\mathrm{best}} \gets \phi(x)$\;
            }
            $p \gets 0$\;
            improved $\gets$ \textbf{true}\;
            \textbf{break}\;
        }
    }

    \If{improved = \textbf{false}}{
        $p \gets p + 1$\;
    }
}

\Return{$x^{\mathrm{best}}$}\;
\end{algorithm}

\section{Experiment setting}
\label{appendix:experiment setting}
This section provides details of dataset generation and training hyperparameters. All our codes and data are available at \href{https://drive.google.com/drive/folders/1P7ixNFALdlZHyGm3vc8eGNRy8681W2jY?usp=drive_link}{LLMOptimizer}.
\subsection{Dataset generation}

For graph generation, we use several standard graph families chosen to span different structural regimes. Specifically, we use \emph{Erd\H{o}s--R\'enyi} graphs~\citep{erdds1959random}, \emph{Barab\'asi--Albert} graphs~\citep{barabasi1999emergence}, \emph{Watts--Strogatz} graphs~\citep{watts1998collective}, random regular graphs, Delaunay graphs, and stochastic block models. We generate \emph{Erd\H{o}s--R\'enyi}, \emph{Barab\'asi--Albert}, \emph{Watts--Strogatz}, and random regular graphs using standard NetworkX implementations~\citep{hagberg2007exploring}.

\paragraph{Graph coloring.}
For graph coloring, we generate in-distribution instances from \emph{Erd\H{o}s--R\'enyi} graphs. Specifically, we generate $G(n,p)$ instances with $n \in [10,300]$ and set $p=d/(n-1)$, where the target average degree satisfies $d \in [3.0,5.2]$. After generation, isolated vertices are removed. To obtain ground-truth labels for 3-coloring, we solve each instance with Google OR-Tools CP-SAT and label it as SAT or UNSAT; for SAT instances, we also extract an explicit coloring. For min-coloring, we solve the optimization problem with Gurobi under a 60-second time limit.

For OOD evaluation, we generate instances from two held-out graph families: \emph{Barab\'asi--Albert} and \emph{Watts--Strogatz}. We again use $n \in [10,300]$ and target average degree $d \in [3.0,5.2]$. To ensure sufficient UNSAT coverage, we implant a $K_4$ clique into a subset of instances. For \emph{Barab\'asi--Albert} graphs, we set $m=\mathrm{round}(\min(d/2,8))$ with $m \ge 1$. For \emph{Watts--Strogatz} graphs, we set $k=\mathrm{round}(d)$ rounded to the nearest even integer with $k \ge 4$, and sample the rewiring probability as $\beta \sim \mathrm{Uniform}[0.2,0.7]$.

\paragraph{Minor-embedding.}
For minor embedding, problem graphs are sampled from four families: \emph{Erd\H{o}s--R\'enyi}, \emph{Barab\'asi--Albert}, \emph{Watts--Strogatz}, and random regular. Specifically, \emph{Erd\H{o}s--R\'enyi} source graphs use edge probability $p \in [0.1,0.6]$; \emph{Barab\'asi--Albert} graphs use attachment parameter $m \in [3,\min(8,n-1)]$; \emph{Watts--Strogatz} graphs use $k \in \{2,4,6,8,10\}$ with $k$ clamped to $n-1$ and forced even, and rewiring probability $p \in [0.05,0.3]$; and random regular graphs use degree $d \in \{4,6,8,10\}$ with parity correction when needed.

Hardware graphs are drawn from a 50\%/50\% mixture of Chimera and random topologies. For Chimera, we fix $t=4$ and vary $(m,n) \in \{1,2,3,4\}^2$, yielding 16 hardware shapes. For random hardware, we sample one family from \emph{Erd\H{o}s--R\'enyi}, random regular, and \emph{Watts--Strogatz}. \emph{Erd\H{o}s--R\'enyi} hardware uses $n \in [20,100]$ and $p \in [0.4,0.7]$, resampling until connected. Random regular hardware uses $n \in [20,100]$ and degree $d \in \{6,8,10,12\}$ with parity correction. \emph{Watts--Strogatz} hardware uses $n \in [20,100]$, $k \in \{6,8,10,12\}$ with $k<n$ and even parity, and rewiring probability $p \in [0.05,0.2]$. For each problem--hardware pair, we sample the source size $|V_P|$ from $[\max(6,\lfloor 0.2|V_H| \rfloor), |V_H|]$, where $|V_H|$ is the size of the hardware graph. We reject and resample any pair violating the trivial lower bounds $|V_H| < |V_P|$ or $|E_H| < |E_P|$.

Since many minor-embedding instances cannot be solved to optimality within a practical time budget, we label them using both exact and heuristic methods. We run Gurobi on the chain-enumeration formulation with maximum chain length $L=3$ and a 60-second time limit. In parallel, we run \emph{minorminer} as a heuristic baseline up to 10 times per instance, or until 60 seconds total, again enforcing chain length at most $L=3$. This gives a dataset in which exact optimization provides infeasibility certificates and quality guarantees when possible, while \emph{minorminer} supplies additional feasible solutions on instances that are harder to solve exactly. Results of this comparison are reported in Table~\ref{table:minor_embedding_compare}.

For OOD evaluation, we introduce both a hardware-topology shift and a problem-graph shift. On the hardware side, we use Pegasus graphs with $m \in \{2,3\}$, which differ structurally from the Chimera and random hardware graphs seen in training. On the problem graph side, we generate graphs from a stochastic block model (SBM). Specifically, we sample the number of source vertices uniformly from $[0.20|V_H|, |V_H|]$, partition them into $B \in [2,5]$ blocks with at least three nodes per block, and sample the intra-block density as $p_{\mathrm{in}} \in [0.20,0.85]$ and the inter-block density as $p_{\mathrm{out}} \in [0.01, \min(0.35, p_{\mathrm{in}}-0.05)]$.

\begin{table}[t]
\centering
\small

\begin{tabular}{lccccc}
\toprule
Method & Inf. Cert. & Feas. & Coverage & Opt. Cert. & Avg. Gap $\downarrow$ \\
\midrule
Gurobi+Zero & 4,816 & 5,281 & 10,097 (50.5\%) & 82.0\% & 0.47\% \\
minorminer & N/A & 6,947 & 6,947 (34.7\%) & N/A & 2.21\% \\
\bottomrule
\end{tabular}
\caption{Comparison of \emph{Gurobi+Zero-Phase Infeasibility Detection} and \emph{minorminer} on 20,000 minor-embedding instances. 
\emph{Inf. Cert.} is the number of instances certified infeasible. 
\emph{Feas.} is the number of feasible solutions found. 
\emph{Coverage} counts all instances resolved by the method, i.e., infeasibility certificates plus feasible solutions. 
\emph{Opt. Cert.} reports the fraction of feasible solutions additionally proved globally optimal. 
\emph{Avg. Gap} is the average relative gap on the 4,842 instances solved by both methods, computed against the best objective found by either method.}
\label{table:minor_embedding_compare}
\end{table}

\subsection{Supervised fine tuning hyperparameters}
We detail the key hyperparameters used for training in table~\ref{table:sft_hparams}.
\begin{table}[t]
\centering
\small
\begin{tabular}{ll}
\toprule
Setting & Value \\
\midrule
Base model & Llama-3 8B \\
Framework & Unsloth + TRL SFTTrainer \\
LoRA rank $r$ & 32 \\
LoRA alpha $\alpha$ & 32 \\
LoRA target modules & $q,k,v,o$, gate, up, down proj \\
Max sequence length & 20000 \\
Optimizer & AdamW \\
Learning rate & $2\times10^{-4}$ \\
Weight decay & $10^{-3}$ \\
LR scheduler & Linear \\
Warmup steps & 20 \\
Per-device train batch size & 2 \\
Gradient accumulation & 4 \\
Effective batch size & 8 \\
Validation split & 10\% \\
Eval / save steps & 200 / 200 \\
Precision & bf16 \\
Random seed & 3407 \\
\bottomrule
\end{tabular}
\caption{Main supervised fine-tuning hyperparameters.}
\label{table:sft_hparams}
\end{table}

\section{Token usage for benchmark models}
\label{appendix: token_usage}
\begin{figure}[htbp]
    \centering

    \begin{subfigure}[b]{0.45\linewidth}
        \centering
        \includegraphics[width=\linewidth]{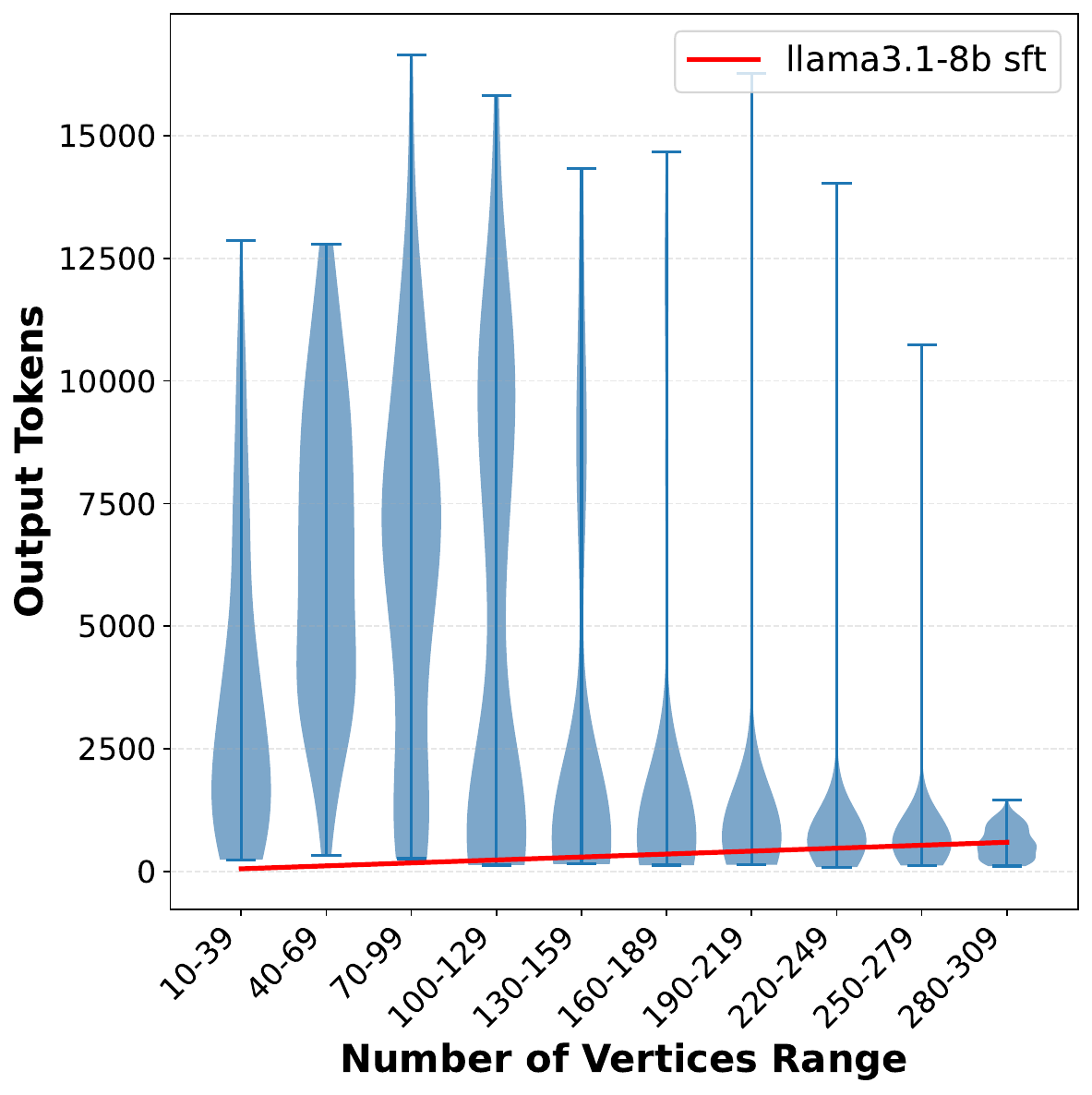}
        \caption{GPT5.2: 3-coloring}
        \label{fig:GPT-3-coloring-token}
    \end{subfigure}
    \hfill
    \hfill
    \begin{subfigure}[b]{0.45\linewidth}
        \centering
        \includegraphics[width=\linewidth]{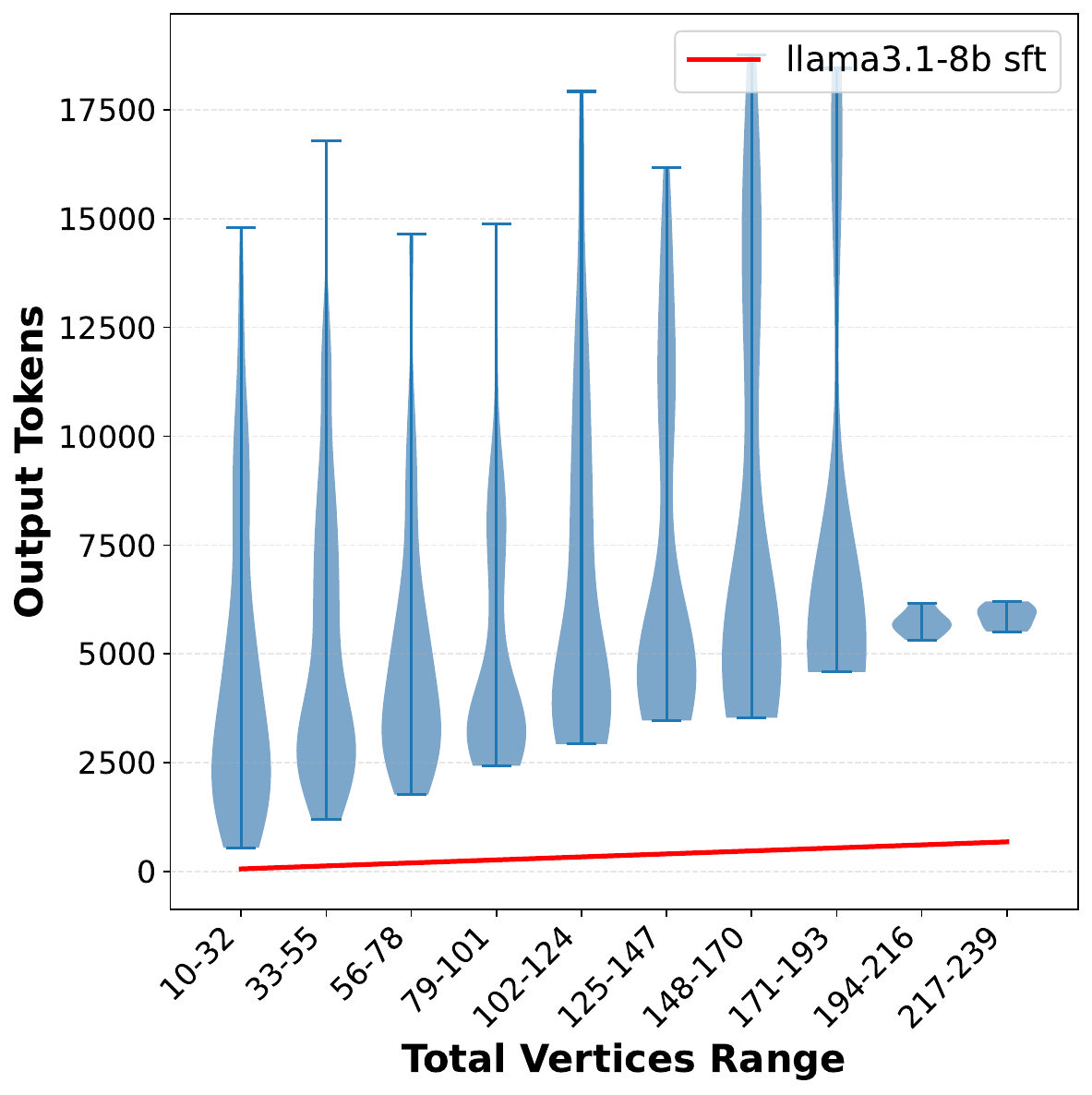}
        \caption{GPT5.2: minor-embedding}
        \label{fig:GPT-embedding-token}
    \end{subfigure}
    \vspace{1em}
    \begin{subfigure}[b]{0.45\linewidth}
        \centering
        \includegraphics[width=\linewidth]{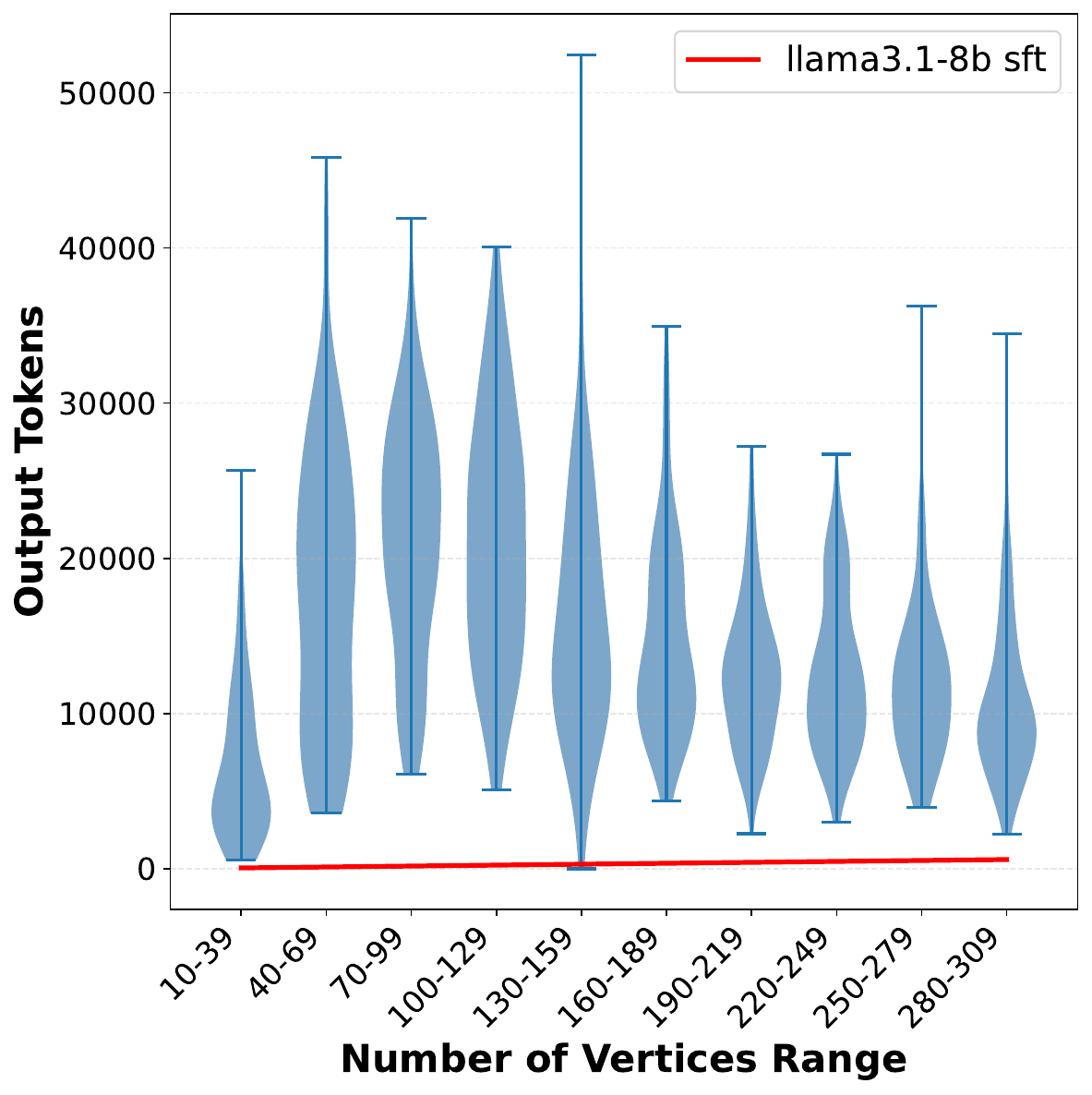}
        \caption{Grok4.1: 3-coloring}
        \label{fig:Grok-3-coloring-token}
    \end{subfigure}
    \hfill
    \begin{subfigure}[b]{0.45\linewidth}
        \centering
        \includegraphics[width=\linewidth]{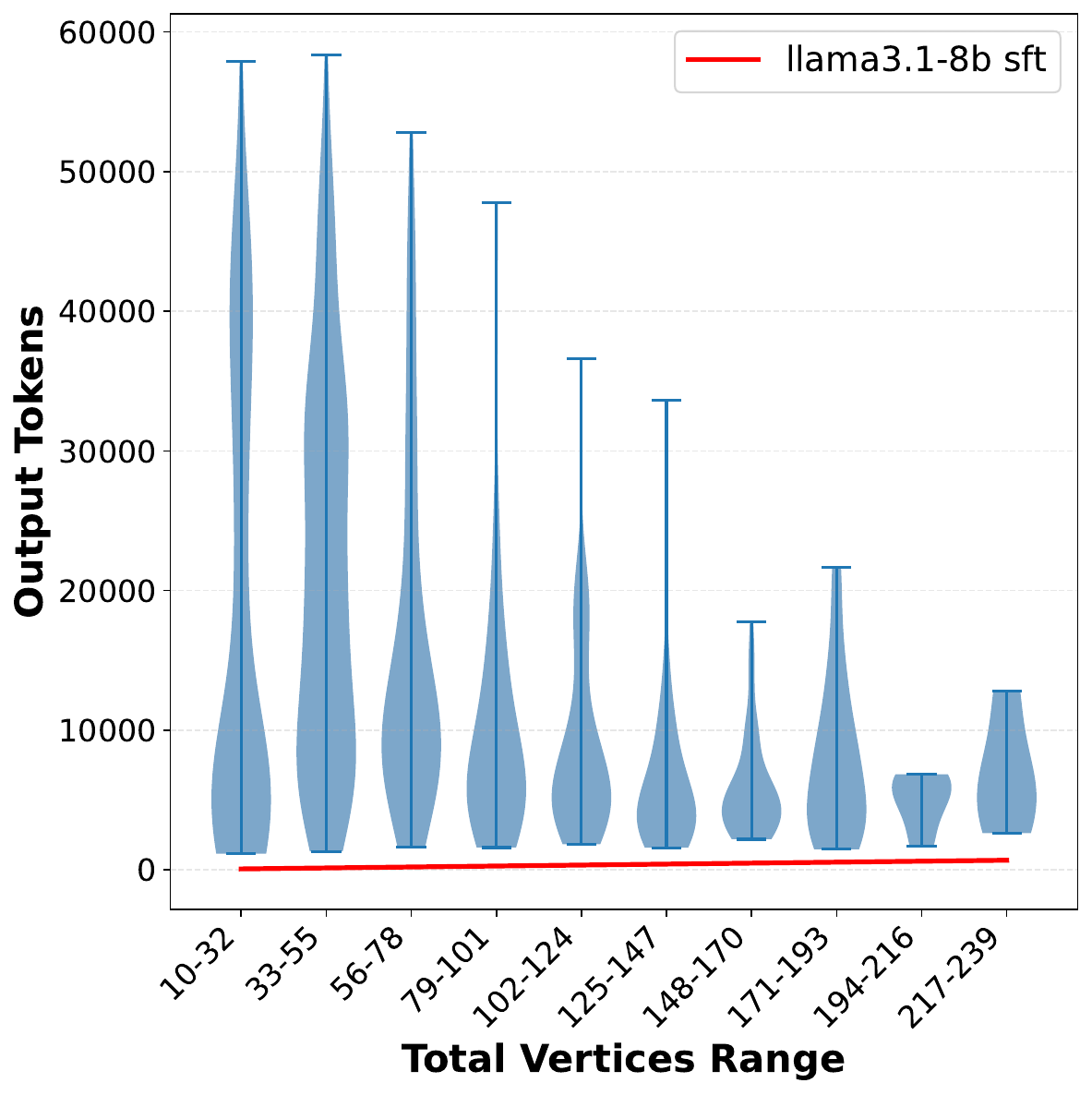}
        \caption{Grok4.1: minor-embedding}
        \label{fig:Grok-embedding-token}
    \end{subfigure}
    \caption{Token usage for GPT-5.2 and Grok-4.1 across tasks. As instance size increases, these models often default to an "infeasible" verdict, bypassing the search for valid solutions to save tokens. In contrast, our fine-tuned model’s token usage scales linearly with problem size due to its structured format.}
    \label{fig:token-usage}
\end{figure}
Figure~\ref{fig:token-usage} compares output token usage across tasks. While our fine-tuned model’s token usage grows predictably with instance size due to its structured format, GPT-5.2 and Grok-4.1 often use fewer tokens on larger instances by prematurely defaulting to an "infeasible" verdict. Although this reduces inference cost, it represents a failure mode where models bypass the constructive search required for feasible solutions. This suggests that the lower token usage for these benchmarks is a byproduct of conservative negative bias rather than genuine efficiency, which aligns with the results in Tables~\ref{tab:minor-embedding-main} and~\ref{table:3-coloring}.

\section{Training dataset example}
\label{appendix:prompt}

\begin{tcolorbox}[enhanced, colback=gray!5!white,colframe=black!75!white,title=Example Training Instance of minor-embedding]

\vspace{1mm}

\textbf{Instruction:} Given a problem graph P with 9 nodes labeled 0..8 and a hardware graph G with 24 nodes, both undirected and given by edge lists, determine whether P can be minor-embedded into G. A valid embedding maps each problem node to a connected chain of hardware nodes, chains for different problem nodes are disjoint, and every problem edge (u,v) must be realized by at least one hardware edge between the two corresponding chains. Limit the chain size up to 3 nodes. Among feasible embeddings, minimize the total number of hardware nodes used. The input also provides, for each node, up to 2 neighbors with highest degree in the form Ni:[a,b,\#c,\#d], where a,b are neighbors and \#c,\#d are their degrees. Output exactly one of the following formats: yes, embedding: \{problem\_node: [hardware\_nodes], ...\}, total nodes used: \{n\_nodes\_used\} or no.

\vspace{1mm}

\textbf{Input:} P edges: [[0,4],[0,5],[0,7],[0,8],[1,2],[1,3],[1,6],[1,8],[2,4],[2,7],[2,8],[3,5],[3,6],[3,7],
[4,6],[4,8],[5,6],[5,7]]
P top2 neighbor-degree info: N0:[4,5,\#4,\#4]; N1:[2,3,\#4,\#4]; N2:[1,4,\#4,\#4]; N3:[1,5,\#4,\#4]; N4:[0,2,\#4,\#4]; N5:[0,3,\#4,\#4]; N6:[1,3,\#4,\#4]; N7:[0,2,\#4,\#4]; N8:[0,1,\#4,\#4]

G edges: [[0,4],[0,5],[0,6],[0,7],[1,4],[1,5],[1,6],[1,7],[2,4],[2,5],[2,6],[2,7],[3,4],[3,5],[3,6],
[3,7],[4,12],[5,13],[6,14],[7,15],[8,12],[8,13],[8,14],[8,15],[9,12],[9,13],[9,14],[9,15],[10,12],
[10,13],[10,14],[10,15],[11,12],[11,13],[11,14],[11,15],[12,20],[13,21],[14,22],[15,23],[16,20],
[16,21],[16,22],[16,23],[17,20],[17,21],[17,22],[17,23],[18,20],[18,21],[18,22],[18,23],[19,20],
[19,21],[19,22],[19,23]]
G top2 neighbor-degree info: N0:[4,5,\#5,\#5]; N1:[4,5,\#5,\#5]; N2:[4,5,\#5,\#5]; N3:[4,5,\#5,\#5]; N4:[12,0,\#6,\#4]; N5:[13,0,\#6,\#4]; N6:[14,0,\#6,\#4]; N7:[15,0,\#6,\#4]; N8:[12,13,\#6,\#6]; N9:[12,13,\#6,\#6]; N10:[12,13,\#6,\#6]; N11:[12,13,\#6,\#6]; N12:[4,20,\#5,\#5]; N13:[5,21,\#5,\#5]; N14:[6,22,\#5,\#5]; N15:[7,23,\#5,\#5]; N16:[20,21,\#5,\#5]; N17:[20,21,\#5,\#5]; N18:[20,21,\#5,\#5]; N19:[20,21,\#5,\#5]; N20:[12,16,\#6,\#4]; N21:[13,16,\#6,\#4]; N22:[14,16,\#6,\#4]; N23:[15,16,\#6,\#4]:

\vspace{1mm}

\textbf{Output:} yes, embedding: {"0": [6,14], "1": [1,7,15], "2": [4], "3": [16,23], "4": [3,5], "5": [18,22], "6": [8,13,21], "7": [10,12,20], "8": [0], total nodes used: 19}

\end{tcolorbox}

\clearpage

\begin{tcolorbox}[colback=gray!5!white,colframe=black!75!white,title=Example Training Instance of 3-Coloring]

\small

\textbf{Instruction:} Given an undirected graph with 12 nodes labeled 0..11 and an edge list, decide whether the graph is 3-colorable. A valid 3-coloring assigns each node $i$ a color $c_i \in \{0,1,2\}$ such that for every edge $(u,v)$, $c_u \neq c_v$. The input also provides, for each node, up to 2 neighbors with highest degree in the form Ni:[a,b,\#c,\#d], where a,b are neighbors and \#c,\#d are their degrees. Output exactly one of: No OR Yes, coloring: [c0,c1,\dots,c(n-1)].

\vspace{1mm}

\textbf{Input:} Edges: [(0,9),(0,10),(0,11),(0,4),(1,11),(1,4),(1,7),(2,6),(2,5),(3,8),(3,9),(4,7),(4,5),(4,10),(5,9),
(5,11),(7,10),(8,9),(9,11)]

N0:[4,9,\#5,\#5]; N1:[4,11,\#5,\#4]; N2:[5,6,\#4,\#1]; N3:[9,8,\#5,\#2]; N4:[0,5,\#4,\#4]; N5:[4,9,\#5,\#5]; N6:[2,\#2]; N7:[4,1,\#5,\#3]; N8:[9,3,\#5,\#2]; N9:[0,5,\#4,\#4]; N10:[4,0,\#5,\#4]; N11:[9,0,\#5,\#4]

\vspace{1mm}

\textbf{Output:} Yes, coloring: [0, 1, 1, 2, 2, 0, 0, 0, 0, 1, 1, 2]

\end{tcolorbox}

\begin{tcolorbox}[ enhanced, colback=gray!5!white,colframe=black!75!white,title=Example Training Instance of Min Coloring]

\vspace{1mm}

\textbf{Instruction:} Given an undirected graph with 48 nodes labeled 0..47 and an edge list, find a coloring that uses the minimum possible number of colors. A valid coloring assigns each node $i$ a color $c_i$ (a nonnegative integer) such that for every edge $(u,v)$, $c_u \neq c_v$. The input also provides, for each node, up to 2 neighbors with highest degree in the form Ni:[a,b,\#c,\#d], where a,b are neighbors and \#c,\#d are their degrees. Output exactly: min\_colors: K, coloring: [c0,c1,\dots,c(n-1)].

\vspace{1mm}

\textbf{Input:} Edges: [(0,44),(0,41),(0,34),(0,43),(0,38),(1,12),(1,23),(1,6),(2,45),(2,6),(2,24),(2,3),
(2,12),(3,40),(3,24),(3,33),(3,45),(4,45),(4,13),(4,31),(4,33),(4,42),(4,14),(5,37),(5,21),(5,15),
(6,11),(6,45),(6,26),(6,12),(7,19),(7,37),(7,30),(7,22),(8,41),(8,22),(8,35),(8,44),(8,30),(9,28),
(9,45),(10,36),(10,40),(10,24),(10,16),(10,41),(11,23),(11,39),(12,29),(12,15),(13,42),(13,14),
(13,18),(13,36),(14,18),(15,25),(15,29),(15,37),(16,40),(16,18),(16,36),(17,23),(17,32),(17,29),
(17,47),(18,36),(18,31),(18,40),(18,33),(19,32),(19,29),(19,22),(20,23),(20,39),(21,37),(21,34),
(21,27),(21,38),(21,30),(22,30),(22,32),(22,47),(23,29),(23,39),(24,25),(24,41),(24,43),(24,38),
(24,40),(25,38),(26,45),(26,42),(27,34),(27,35),(27,46),(28,42),(29,32),(29,37),(32,47),(33,40),
(34,46),(35,44),(35,46),(39,42),(39,47),(41,44),(41,43),(44,46)]

N0:[41,44,\#6,\#5]; N1:[6,23,\#6,\#6]; N2:[24,6,\#8,\#6]; N3:[24,40,\#8,\#6]; N4:[45,13,\#6,\#5]; N5:[21,15,\#6,\#5]; N6:[45,2,\#6,\#5]; N7:[22,37,\#6,\#5]; N8:[22,41,\#6,\#6]; N9:[45,28,\#6,\#2]; N10:[24,40,\#8,\#6]; N11:[6,23,\#6,\#6]; N12:[29,6,\#7,\#6]; N13:[18,4,\#7,\#6]; N14:[18,4,\#7,\#6]; N15:[29,12,\#7,\#5]; N16:[18,40,\#7,\#6]; N17:[29,23,\#7,\#6]; N18:[40,13,\#6,\#5]; N19:[29,22,\#7,\#6]; N20:[23,39,\#6,\#5]; N21:[37,27,\#5,\#4]; N22:[8,32,\#5,\#5]; N23:[29,39,\#7,\#5]; N24:[40,41,\#6,\#6]; N25:[24,15,\#8,\#5]; N26:[6,45,\#6,\#6]; N27:[21,34,\#6,\#4]; N28:[42,9,\#5,\#2]; N29:[23,12,\#6,\#5]; N30:[21,22,\#6,\#6]; N31:[18,4,\#7,\#6]; N32:[29,22,\#7,\#6]; N33:[18,4,\#7,\#6]; N34:[21,0,\#6,\#5]; N35:[8,44,\#5,\#5]; N36:[18,10,\#7,\#5]; N37:[29,21,\#7,\#6]; N38:[24,21,\#8,\#6]; N39:[23,42,\#6,\#5]; N40:[24,18,\#8,\#7]; N41:[24,0,\#8,\#5]; N42:[4,13,\#6,\#5]; N43:[24,41,\#8,\#6]; N44:[41,0,\#6,\#5]; N45:[4,6,\#6,\#6]; N46:[44,27,\#5,\#4]; N47:[22,32,\#6,\#5]

\vspace{1mm}

\textbf{Output:} min\_colors: 4, coloring: [3, 0, 1, 0, 0, 0, 3, 1, 0, 3, 0, 0, 2, 1, 3, 1, 2, 0, 0, 0, 0, 3, 3, 2, 3, 0, 0, 1, 0, 3, 2, 3, 2, 3, 2, 3, 3, 2, 2, 3, 1, 2, 2, 0, 1, 2, 0, 1]

\end{tcolorbox}

\section{Model-then-solve details}
\label{appendix:model-then-solve}
We provide the prompt used to instruct LLM to formulate the minor-embedding problem:
\clearpage
\begin{tcolorbox}[
enhanced, colback=gray!5!white,colframe=black!75!white,
    title=Model-then-solve Prompt,
]
You are an expert in Operations Research, Mixed-Integer Linear Programming (MILP), and graph theory. Your task is to mathematically formulate and solve a Graph Minor-Embedding problem using Python and the \texttt{gurobipy} library.

\medskip
\textbf{Problem Description}

We are given two undirected graphs:
\begin{enumerate}
    \item A Problem Graph, $P$, with \texttt{n} nodes labeled \texttt{0} to \texttt{n-1}, and a set of edges $E_P$.
    \item A Hardware Graph, $G$, with \texttt{m} nodes labeled \texttt{0} to \texttt{m-1}, and a set of edges $E_G$.
\end{enumerate}

We want to find a valid minor-embedding of $P$ into $G$. A valid embedding maps each problem node $i$ in $P$ to a subset of hardware nodes $S_i$ in $G$ (called a ``chain''), subject to the following rules:
\begin{enumerate}
    \item \textbf{Disjoint Chains:} For any two distinct problem nodes $i$ and $j$, their corresponding chains must not overlap (i.e., $S_i \cap S_j = \emptyset$).
    \item \textbf{Chain Connectivity:} For every problem node $i$, the subgraph induced by $S_i$ in $G$ must be connected.
    \item \textbf{Chain Size Limit:} The maximum number of hardware nodes in any chain $S_i$ is strictly limited to 3 (i.e., $|S_i| \leq 3$).
    \item \textbf{Edge Realization:} For every problem edge $(i,j)$ in $E_P$, there must exist at least one hardware edge $(u,v)$ in $E_G$ such that $u \in S_i$ and $v \in S_j$.
\end{enumerate}

\textbf{Objective}

Among all feasible valid embeddings, minimize the total number of hardware nodes utilized across all chains (i.e., minimize the sum of $|S_i|$ for all $i$).

\medskip
\textbf{Task 1: Mathematical Formulation}

Write a complete, rigorous Mixed-Integer Linear Programming (MILP) formulation for this problem. You must strictly provide:
\begin{itemize}
    \item \textbf{Sets and Indices:} Clearly define all sets (e.g., nodes, edges).
    \item \textbf{Parameters:} Define any given constants.
    \item \textbf{Decision Variables:} Clearly define all binary/integer/continuous variables.
    \item \textbf{Objective Function:} The mathematical expression to minimize.
    \item \textbf{Constraints:} All rules stated above translated into linear inequalities. Provide brief, logical explanations for each constraint block.
\end{itemize}

\textbf{Task 2: Python + Gurobi Implementation}

Write a self-contained Python script using the \texttt{gurobipy} library to solve this model.
\begin{itemize}
    \item Encapsulate the logic in a function with the following signature:
    \begin{center}
    \texttt{def solve\_minor\_embedding(n: int, P\_edges: list, m: int, G\_edges: list) -> dict:}
    \end{center}
    \item The function should return a dictionary mapping each problem node to its list of hardware nodes (e.g., \texttt{\{0: [1, 2], 1: [3], ...\}}). If the problem is infeasible, return \texttt{None}.
    \item Ensure the code strictly follows your mathematical model.
    \item Comment the code to show which block of constraints corresponds to which mathematical equation.
    \item Include a small, reproducible \texttt{if \_\_name\_\_ == "\_\_main\_\_":} block with dummy data to test the function.
\end{itemize}
\end{tcolorbox}

The full response from ChatGPT is available at the following link: \href{https://drive.google.com/drive/folders/1P7ixNFALdlZHyGm3vc8eGNRy8681W2jY?usp=drive_link}{LLMOptimizer}.

\end{document}